\documentclass{article} % For LaTeX2e
\usepackage{mystyle}
\usepackage{amsmath,amsfonts,bm}

\def\eqref#1{equation~\ref{#1}}
\def\1{\bm{1}}

\DeclareMathAlphabet{\mathsfit}{\encodingdefault}{\sfdefault}{m}{sl}
\SetMathAlphabet{\mathsfit}{bold}{\encodingdefault}{\sfdefault}{bx}{n}

\DeclareMathOperator*{\argmax}{arg\,max}
\DeclareMathOperator*{\argmin}{arg\,min}

\usepackage{fullpage}
\usepackage{hyperref}
\usepackage{url}

\usepackage{dsfont}
\usepackage{url}
\usepackage{fourier}
\usepackage{mathtools}
\usepackage{amsmath}
\usepackage[ruled,vlined]{algorithm2e}
\allowdisplaybreaks

\newcommand{\iid}{\overset{\mathrm{iid}}{\sim}}

\title{Compressed online Sinkhorn}

\author{
Fengpei Wang \\
University of Bath \\
\texttt{fw468@bath.ac.uk}
\thanks{Supported by a scholarship from China Scholarship Council and the EPSRC Centre for Doctoral Training in Statistical Applied Mathematics at Bath (SAMBa), under the project EP/S022945/1}
\and
Clarice Poon \\
University of Warwick \\
\texttt{clarice.poon@warwick.ac.uk}
\and
Tony Shardlow  \\
University of Bath \\
\texttt{t.shardlow@bath.ac.uk}
}

\begin{document}

\maketitle

\begin{abstract}

The use of optimal transport (OT) distances, and in particular entropic-regularised OT distances, is an increasingly popular evaluation metric in many areas of machine learning and data science. Their use has largely been driven by the availability of efficient algorithms such as the Sinkhorn algorithm. One of the drawbacks of the Sinkhorn algorithm for large-scale data processing is that it is a two-phase method, where one first draws a large stream of data from the probability distributions, before applying the Sinkhorn algorithm to the discrete probability measures. More recently, there have been several works developing stochastic versions of Sinkhorn that directly handle continuous streams of data. In this work, we revisit the recently introduced \textit{online Sinkhorn algorithm} of \cite{mensch2020online}. Our contributions are twofold: We improve the convergence analysis for the online Sinkhorn algorithm, the new rate that we obtain is faster than the previous rate under certain parameter choices. We also present numerical results to verify the sharpness of our result. Secondly, we propose the \textit{compressed online Sinkhorn algorithm} which combines measure compression techniques with the online Sinkhorn algorithm. We provide numerical experiments to show practical numerical gains, as well as theoretical guarantees on the efficiency of our approach.

%Optimal transport (OT) is the coupling between two probability distributions with minimal transport cost, where the cost is usually determined by distance. The Sinkhorn algorithm is an efficient algorithm for solving the entropic regularised OT problem between two discrete distributions. However, the non-discrete case does not yet have a widely accepted solution. Mensch and Peyr{\'e} in \cite{mensch2020online} proposed the online Sinkhorn algorithm that uses incoming samples from the non-discrete distributions at every step.
% contribution
%We propose our original Compressed online Sinkhorn that improves the complexity with respect to the error of online Sinkhorn in \cite{mensch2020online}. We apply the Fourier compression method that has not been studied before according to our knowledge. We also spot a mistake in a proof in the above paper and correct it, which leads to a new error convergence rate.
%
% overview of the paper
%In this paper, we first give an overview of the online Sinkhorn algorithm and prove the error convergence rate in the variational norm. We then compare our numerical results to our theoretical results and the theoretical results proved in \cite{mensch2020online}. Finally, we show the complexity improvement of our original compressed online Sinkhorn algorithm and show the numerical convergence of it compared to online Sinkhorn.
\end{abstract}

\section{Introduction}
% background

A fundamental problem in data processing is the computation of metrics or distances to effectively compare different objects of interest \cite{peyre2019computational}. In the last decade, it has become apparent that many problems, including image processing \cite{ferradans2014regularized,ni2009local}, natural language processing \cite{xu2020vocabulary} and genomics \cite{schiebinger2019optimal}, can be modelled using probability distributions and optimal transport (OT) or \emph{Wasserstein distances} have become widely adopted as evaluation metrics. The use of such distances have become especially prevalent in the machine learning community thanks to the vast amount of research in computational aspects of entropic-regularised optimal transport \cite{cuturi2013sinkhorn,peyre2019computational}. The use of entropic-regularised optimal transport has been especially popular since they can be easily computed using the celebrated Sinkhorn algorithm \cite{cuturi2013sinkhorn}, and such distances are known to have superior statistical properties, circumventing the curse of dimensionality \cite{genevay2019sample}.

Given the wide-spread interest in computational optimal transport and in particular, entropic regularised distances in large data-processing applications, there have been several lines of work on extending the Sinkhorn algorithm to handle large-scale datasets. The application of the Sinkhorn algorithm typically involves first drawing samples from the distributions of interest and constructing a large pairwise distance matrix, then applying the Sinkhorn algorithm to compute the distance between the sampled empirical distributions. While there have been approaches to accelerate the second step with Nystr{\"o}m compression \cite{altschuler2019massively} or employing greedy approaches \cite{altschuler2017near}, in recent years, there has been an increasing interest in the development of online versions of Sinkhorn that can directly compute OT distances between continuous distributions. One of the computational challenges for computing the OT distance between continuous distributions is that the dual variables (Kantorovich potentials) are continuous functions and one needs to represent these functions in a discrete manner. Two of the main representations found in the literature include the use of reproducing kernel Hilbert spaces \cite{aude2016stochastic} and more recently, the online Sinkhorn algorithm \cite{mensch2020online} was introduced where the Kantorovich potentials are represented using sparse measures and special kernel functions that exploit the particular structure of OT distances.

\paragraph{Contributions}
In this work, we revisit the online Sinkhorn algorithm of \cite{mensch2020online}: we improve their theoretical convergence result for this method, and propose a compressed version of this method to alleviate the large memory footprint of this method. 
Our contributions can be summarised as follows
\begin{itemize}
\item We provide an updated theoretical convergence rate that is under certain parameter choices is faster than the rate given in \cite{mensch2020online}. We also numerically verify that our theoretical analysis is sharp.
\item We propose a compressed version of the online Sinkhorn algorithm. The computational complexity of the online Sinkhorn algorithm grows polynomially with each iteration and it is natural to combine compression techniques with the online Sinkorn algorithm. As explained in Section \ref{sec:COS}, the online Sinkhorn algorithm seeks to represent the continuous Kantorovich potentials as measures (super-position of Diracs) and there are some popular methods for measure compression such as kernel recombination \cite{cosentino2020randomized,hayakawa2022positively,adachi2022fast}) and Nystr{\"o}m method \cite{zhang2008improved}. % There is also kernel herding \cite{chen2012super}, kernel thinning \cite{dwivedi2021kernel} and Stein thinning \cite{riabiz2022optimal} for compressing the measures.
We apply compression with a Fourier-based moments approach. We present theoretical complexity analysis and numerical experiments to show that our approach can offer significant computational benefits.

\end{itemize}

\section{Online Sinkhorn}

%Let $\Xx$ be a compact subset of $\mathbb{R}^d$ and $\mathcal{C}(\mathcal X)$ denote the set of continuous functions  $\mathcal X\to\mathbb{R}$. 
%Let $\al,\be\in\Pp(\Xx)$, where $\Pp(\Xx)$ is the set of probability measures on $\Xx$. Let  $C\colon\Xx\times\Xx\to\mathbb{R}$ be a cost function between two points $x,y\in \Xx$. 

\subsection{The Sinkhorn algorithm}
Let $\Xx$ be a compact subset of $\mathbb{R}^d$ and $\mathcal{C}(\mathcal X)$ denote the set of continuous functions $\mathcal X\to\mathbb{R}$.
The Kantorovich formulation was first proposed to study the transport plan between two probability distributions $\alpha,\beta$ with the minimal cost \cite{kantorovich1940on}:
\begin{equation}\label{eq:OT}
    \min_{\pi\in \Pi(\alpha,\beta)} \int C(x,y) \mathrm{d}\pi(x,y),
\end{equation}
where $\Pi(\alpha,\beta)$ is the set of positive measures with fixed marginals $\alpha$ and $\beta$,
\begin{equation}
    \Pi(\alpha,\beta) \eqdef \ens{\pi\in\Pp(\Xx^2)\colon\alpha=\int_{y\in\Xx}\mathrm{d}\pi(\cdot,y),\; \beta =\int_{x\in\Xx}\mathrm{d}\pi(x,\cdot)},
\end{equation}
and
$C\colon\Xx\times\Xx\to\mathbb{R}$ is a given cost function. 

The solution to the optimisation problem \eqref{eq:OT} above can be approximated by a strictly convex optimisation problem by adding a regularisation term: The entropic regularised OT problem with the regularisation parameter $\epsilon>0$ is
\begin{equation}\label{eq:regulariedOT}
\underset{\pi\in \Pi(\alpha,\beta)}{\min} \int C(x,y) \mathrm{d}\pi(x,y) + \epsilon\text{KL}\pa{\pi, \alpha \otimes \beta},
\end{equation}
where $\al\otimes\be$ is the product measure on $\Xx^2$, and $\text{KL}\pa{\pi, \alpha\otimes \beta}\eqdef\int\log\pa{\frac{\mathrm{d}\pi}{\mathrm{d}\al\otimes\be}}\mathrm{d}\pi$ is the Kulback--Leibler divergence \cite{cuturi2013sinkhorn}.

The following maximisation problem for $\epsilon>0$ is a dual formulation of the entropic regularised OT problem \eqref{eq:regulariedOT}:
\begin{equation*}
    F_{\al,\be}(f,g) \coloneqq \underset{f,g\in \Cc\pa{\Xx}}{\max} \int f(x)d\al(x) + \int g(y)d\be(y) - \epsilon \int e^{\frac{f\oplus g - C}{\epsilon}} d\al(x)d\be(y) 
\end{equation*}
where $\pa{f\oplus g - C}\pa{x,y} \coloneqq f(x)+g(y)-C(x,y)$ and $\pa{f,g}$ are defined to be the pair of dual potentials \cite{peyre2019computational}.

The Sinkhorn algorithm works by alternating minimization on the dual problem $F_{\al,\beta}$, and is for discrete distributions. So, to apply Sinkhorn, one first draws empirical distributions $\alpha_n \eqdef \frac1n \sum_{i=1}^n \delta_{x_i}$ and $\beta_n \eqdef  \frac1n\sum_{j=1}^n\delta_{y_i}$ with $x_i\iid \alpha$ and $y_i\iid \beta$, then compute iteratively:
\begin{equation}\label{eq:disc_sink}
\begin{split}
u_{t+1} =  \pa{   \frac1n \sum_{j=1}^n \frac{1}{v_{t,j}} \exp\pa{ -\frac{C(x_k,y_j)}{\epsilon} }}_{k=1}^n ,\quad
v_{t+1} =   \pa{ \frac1n \sum_{i=1}^n \frac{1}{u_{t+1,i}} \exp\pa{\frac{-C\pa{x_i,y_k}}{\epsilon}}}_{k=1}^n, 
\end{split}
\end{equation}
where $u_t \eqdef \exp\pa{-f_t/\epsilon} \in \RR^n$ and $v_t \eqdef \exp\pa{-g_t/\epsilon} \in\RR^n$ and $f_t,g_t\in \RR^n$ are the pair of dual potentials at step $t$. The computational complexity is $O\pa{n^2\log\pa{n}\delta^{-3}}$ for reaching 
$\delta$ accuracy \cite{altschuler2017near,mensch2020online}.

%In the continuous setting, one first needs to draw a large number of samples from the measure $\alpha$ and $\beta$, and then apply the Sinkhorn algorithm on them, as usually the integration with respect to $\alpha$ and $\beta$ are not feasible. 

\subsection{The online Sinkhorn algorithm}

In the continuous setting, the Sinkhorn iterations operate on functions $u_t, v_t \in \Cc(\Xx)$ and involve the full distributions $\al,\beta$:
\begin{equation}\label{eq:cont_sink}
u_{t+1} = \exp\pa{-\frac{f_{t+1}}{\epsilon}} =     \int \frac{1}{v_{t}(y)} K_y(\cdot)d\beta(y),\quad 
v_{t+1} = \exp\pa{-\frac{g_{t+1}}{\epsilon}} =    \int \frac{1}{u_{t+1}(x)} K_x(\cdot) d\alpha(x), 
\end{equation}
where $K_y(\cdot) = \exp\pa{-\frac{C(\cdot,y)}{\epsilon}}$ and $K_x(\cdot) = \exp\pa{-\frac{C(x,\cdot)}{\epsilon}}$. In \cite{mensch2020online}, a natural extension of the Sinkhorn algorithm was proposed, replacing $\alpha$ and $\beta$ at each step with empirical distributions of growing supports $\hat \alpha_t =\frac{1}{n} \sum_{i=n_t}^{n_{t+1}} \delta_{x_i}$ and $\hat \beta_t =\frac{1}{n} \sum_{i=n_t+1}^{n_{t+1}} \delta_{y_i}$ where $x_i\iid \alpha$ and $y_i\iid \beta$. For appropriate learning rate $\eta_t$, the online Sinkhorn iterations \footnote{In the original Sinkhorn algorithm, the $u_{t+1}$ in \eqref{eq:2stochastic_online_updates} is replaced with $u_t$, but we consider $u_{t+1}$ in this paper to better match with the classical Sinkhorn algorithm. This makes little difference to the analysis.} are defined as
\begin{align}
u_{t+1} &= \pa{1-\eta_t}u_{t} + \eta_t \int \frac{1}{v_t(y)} K_y(\cdot) \mathrm{d}\Hat\be_t(y),\label{eq:1stochastic_online_updates}\\
v_{t+1} &= \pa{1-\eta_t}v_{t} + \eta_t \int \frac{1}{u_{t+1}(x)} K_x(\cdot) \mathrm{d}\Hat\al_t(x).\label{eq:2stochastic_online_updates}
\end{align}

The key observation in \cite{mensch2020online} is that the continuous functions $u_t$ and $v_t$ can be discretely represented using vectors $(q_{i,t}, y_i)$ and $(p_{i,t},x_i)$, in particular, $u_t$ and $v_t$ take the following form:
\begin{equation}\label{eq:Sinkhorn_weights_representation}
    u_t  = \sum_{i=1}^{n_t} \exp\pa{\frac{q_{i,t} - C(\cdot, y_i)}{\epsilon}}\qandq
    v_t  = \sum_{i=1}^{n_t} \exp\pa{\frac{p_{i,t} - C(x_i,\cdot)}{\epsilon}},
\end{equation}
where $q_{i,t},p_{i,t}$ are weights and $\pa{x_i,y_i}$ are the positions, for $1\leq i \leq n_t$. Thanks to this representation on $u_t, v_t$, the online Sinkhorn algorithm only needs to record and update vectors $(p_{i,t}, x_i)_i$ and $(q_{i,t}, y_i)_i$. 
The algorithm is summarised in Algorithm \ref{alg:online_sinkhorn}.

\begin{algorithm}[H]
    \label{alg:online_sinkhorn}
    \DontPrintSemicolon
    \caption{Online Sinkhorn (Mensch and Peyr\'e, 2020)}
    \textbf{Input:} Distributions $\al$ and $\be$, learning weights $(\eta_t)_t$, batch sizes $\pa{b_t}_t$\\
    \textbf{Set} $p_i = q_i =0$ for $i \in (0,n_0]$, where $n_0=b_0$\\
     \For{$t = 0,\cdots, T-1$}{
     \begin{enumerate}
         \item Sample $(x_i)_{(n_t,n_{t+1}]} \iid \al$, $(y_i)_{(n_t,n_{t+1}]} \iid \be$, where $n_{t+1}=n_t+b_{t}$
         \item Evaluate $(g_t(y_i))_{i = (n_t,n_{t+1}]}$ via \eqref{eq:Sinkhorn_weights_representation}\label{Sinkhorn:step2}
         \item $q_{(n_t,n_{t+1}],t+1} \leftarrow \epsilon\log(\frac{\eta_t}{b_{t+1}}) + \pa{g_t(y_i)}_{(n_t,n_{t+1}]}$ \label{Sinkhorn:step3}
         \item $q_{(0,n_t],t+1} \leftarrow q_{(0,n_t],t} + \epsilon\log(1-\eta_t)$ 
         \item Evaluate $(f_{t+1}(x_i))_{i = (n_t,n_{t+1}]}$ via \eqref{eq:Sinkhorn_weights_representation} \label{Sinkhorn:step5}
         \item $p_{(n_t,n_{t+1}],t+1} \leftarrow \epsilon\log(\frac{\eta_t}{b_{t+1}}) + \pa{f_{t+1}(x_i)}_{(n_t,n_{t+1}]}$ 
         \item $p_{(0,n_t],t+1} \leftarrow p_{(0,n_t],t} + \epsilon\log(1-\eta_t)$
     \end{enumerate}
     }
     \textbf{Returns:} ${f}_T:(q_{i,T},y_i)_{(0,n_T]}$ and ${g}_T:(p_{i,T},x_i)_{(0,n_T]}$
\end{algorithm}

\subsubsection{Convergence analysis for online Sinkhorn}\label{sec:OS_complexity}

Following \cite{mensch2020online}, convergence can be established under the following three assumptions.
\begin{ass}\label{assumption1}
    The cost $C:\Xx \times \Xx \to \mathbb{R}$ is $L$-Lipschitz continuous for some Lipschitz constant $L>0$.
\end{ass}

\begin{ass}\label{assumption2}
    $(\eta_t)_t$ is such that $\sum_{t=1}^{\infty} \eta_t = \infty$ and $\sum_{t=1}^{\infty} \eta_t^2 < \infty$, where $0 < \eta_t < 1$ for all $t>0$.
\end{ass}

\begin{ass}\label{assumption3}
    $\pa{b_t}_t$ and $\pa{\eta_t}_t$ satisfy that $\sum_{t=1}^{\infty}\frac{\eta_t}{\sqrt{b_t}} < \infty$.
\end{ass}

In order to satisfy Assumptions \ref{assumption2} and \ref{assumption3}, from now on, we let $-1<b<-\frac{1}{2}$ and $a-b>1$ and take $$\eta_t =  t^{b} \qandq b_t = t^{2a}.$$

 For the online Sinkhorn algorithm, we obtain the following convergence result:
% the derivations of Moulines and Bach \cite{moulines2011non}. Instead of $O\pa{t^{b}}$, the upper bound should be $O\pa{t^{-a}}$ according to Lemma \ref{lem:summimg_up_e_t}. As a result of this, in \cite{mensch2020online} it is shown that $\delta_N=O\pa{\mathcal{O}\pa{N^{-\frac{a}{2a+1}}}}$, which can only be arbitrarily close to $O\pa{\frac{1}{\sqrt{N}}}$.
%

\begin{thm}\label{thm:OSerror}  Let $f^*$ and $g^*$ denote the optimal potentials. Let ${f}_t$ and ${g}_t$ be the output of Algorithm \ref{alg:online_sinkhorn} after $t$ iterations.
    Suppose $\eta_t =  t^{b}$ for $-1<b<-\frac{1}{2}$, and $b_t = t^{2a}$ with $a-b>1$. For a constant $c>0$,
\begin{equation}\label{eq:errorbound}
        \delta_N \lesssim  \exp\pa{-c N^\frac{b+1}{2a+1}}+N^{-\frac{a}{2a+1}} = {O}\pa{N^{-\frac{a}{2a+1}}},
    \end{equation}
    where
 $$\delta_N \eqdef \norm{\pa{ f_{t\pa{N}} - f^*}/\epsilon}_{var} + \norm{ \pa{ g_{t\pa{N}} - g^*}/\epsilon}_{var},$$ and ${t\pa{N}}$ is the first iteration number for which $\sum_{i=1}^t b_i >N$.
\end{thm}

\paragraph{Comparison with the previous rate}
It was shown in \cite[Proposition 4]{mensch2020online} that
 $$
 \delta_N\lesssim \exp\pa{-c N^\frac{b+1}{2a+1}} +N^{\frac{b}{2a+1}}.
 $$
 By taking $b$ close to $-1$ and $a$ close to 0, the asymptotic convergence rate can be made arbitrarily close to $\Oo\pa{N^{-1}}$. There however seems to be an error in the proof (see comment after the sketch proof and the appendix). In contrast, our asymptotic convergence rate is ${O}\pa{N^{-\frac{a}{2a+1}}}$, which is at best $\Oo(N^{-1/2})$. However, our convergence rate is faster whenever
$
a\geq -b
$. See Figure \ref{fig:rates_1}.

% \paragraph{Difference in convergence rate}
% Ignoring the exponential term in \eqref{eq:errorbound}, our convergence rate is $\mathcal{O}\pa{N^{-\frac{a}{2a+1}}}$ at best, the convergence rate is $O(N^{-1/2})$ when $a\to \infty$. In practice, choosing $a$ large means the transient term $\exp\pa{-C N^\frac{b+1}{2a+1}}$ dominates the convergence rate until the number of iterates $t$ is large enough, and we won't see the asymptotic rate. Note that our convergence rate is faster whenever
% $
% a\geq -b
% $.

\paragraph{Sketch of proof}
Let $U_t \eqdef \pa{ f_t - f^*}/\epsilon$, $V_t \eqdef \pa{g_t - g^*}/\epsilon$, $e_t \eqdef \|U_t\|_{var} + \|V_t\|_{var}$, $\zeta_{\Hat\be_t} \eqdef \pa{T_{\be}(g_t) - T_{\Hat\be_t}(g_t)}/\epsilon$ and $\iota_{\Hat\al_t}\eqdef \pa{T_\al(f_{t+1})- T_{\Hat\al_t}(f_{t+1})}/\epsilon$, where $T_{\mu}$ is defined in Appendix \ref{subsec:lemmas}. It can be shown that
$$
    e_{t+1} \leq \pa{ 1-\eta_t + \eta_t \kappa }e_t +\eta_t\pa{\|\zeta_{\be_t}\|_{var} +\|\iota_{\al_t}\|_{var}},
$$
and
$$
        \mathbb{E}\norm{\zeta_{\Hat\be_t}}_{var}\lesssim \frac{A_1(f^*,g^*,\epsilon)}{\sqrt{b_t}} \qandq 
    \mathbb{E}\norm{\iota_{\Hat\al_t}}_{var}\lesssim \frac{A_2(f^*,g^*,\epsilon)}{\sqrt{b_t}},
$$
where $A_1(f^*,g^*,\epsilon),A_2(f^*,g^*,\epsilon)$ are constants depending on $f^*,g^*,\epsilon$.

Denote $S= A_1(f^*,g^*,\epsilon)+A_2(f^*,g^*,\epsilon)$, then taking expectations
\begin{equation}
        \mathbb{E}e_{t+1} \lesssim \pa{ 1-\eta_t + \eta_t \kappa }\mathbb{E} e_t+ \frac{S\eta_t}{\sqrt{b_t}}.
\end{equation}

Applying the Gronwall lemma \cite[Lemma 5.1]{gronwell-notes} to this recurrence relation, we get
\begin{equation}\label{eq:sketch_err_its}
    \mathbb{E} e_{t+1} \lesssim \pa{\mathbb{E} e_1 + \frac{S}{a-b-1}}\exp\pa{\frac{\kappa-1}{b+1} t^{b+1}} + t^{-a}.
\end{equation}
The total number of samples $N=O(t^{(2a+1)})$ and this substitution completes the proof.
\hfill\qedsymbol{}

In \cite{mensch2020online}, the last term of the upper bound in \eqref{eq:sketch_err_its} was calculated as $t^b$, but this followed the derivations of \cite[Theorem 2]{moulines2011non} that do not directly apply, which leads to the error in Theorem \ref{thm:OSerror} to be $O\pa{N^{\frac{b}{2a+1}}}$. 

\subsubsection{Numerical verification of our theoretical rate}\label{sec:OS_numerical}
To show that the error rates in Theorem \ref{thm:OSerror} are correct, we plot the variational error for online Sinkhorn and display our theoretical convergence rate with exponent $\frac{-a}{2a+1}$ and the `old theoretical' convergence rate with exponent ${\frac{b}{2a+1}}$ from \cite[Proposition 4]{mensch2020online}.

% we compared the numerical results from running online Sinkhorn with $O\pa{\frac{-a}{2a+1}}$ in Theorem \ref{thm:OSerror} and $O\pa{\frac{b}{2a+1}}$ shown in \cite[Proposition 4]{mensch2020online}. We expect to see two sequences with the same order of convergence rates to be plotted as two parallel lines in the log-log plot.

 Our experiments are in $1$D, $2$D and $5$D. 
 % For each dimension, we test both the OT problem of two continuous distributions.
 For the $1$D case, the source and target distributions are the Gaussian distributions $\Nn\pa{3,4}$ and $\Nn\pa{1,2}$.
%We tested our methods in $1$D continuous case on two Gaussian distributions $\Nn\pa{3,4}$ and $\Nn\pa{1,2}$.
In $2$D, the source distribution $\Nn\pa{\mu^{(1)},\Sigma^{(1)}}$ is generated by $\mu_i^{(1)}\sim \Uu\pa{0,10}$, for $i=1,2$ and $\Sigma^{(1)}$ is a randomly generated covariance matrix. The target distribution $\Nn\pa{\mu^{(2)},\Sigma^{(2)}}$ is generated by $\mu_i^{(2)}\sim \Uu\pa{0,5}$, for $i=1,2$ and $\Sigma^{(2)}$ is a randomly generated covariance matrix. 
The $y$-axis in all the plots in this section are the errors of the potential functions $\norm{f_t-f_{t-1}}_{var} + \norm{g_t-g_{t-1}}_{var}$ based on a finite set of points $x_i,y_i$. The plots show that the convergence behaviour of online Sinkhorn asymptotically is parallel to our theoretical line, which corroborates with Theorem \ref{thm:OSerror}.

\begin{figure}[htp]
     \centering
     \begin{tabular}{ccc}
          \includegraphics[height=3.3cm]{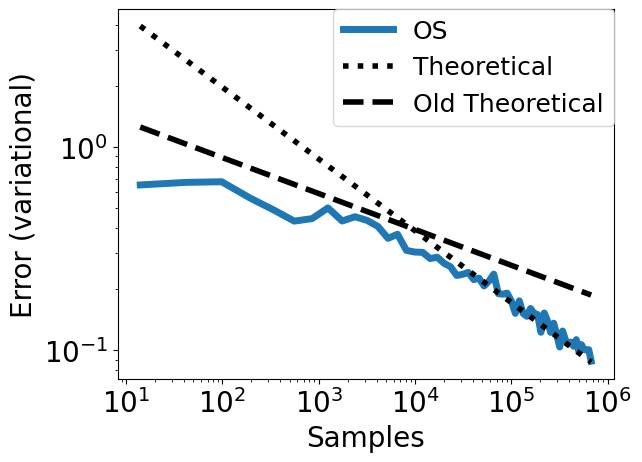}
\!
     \includegraphics[height=3.3cm]{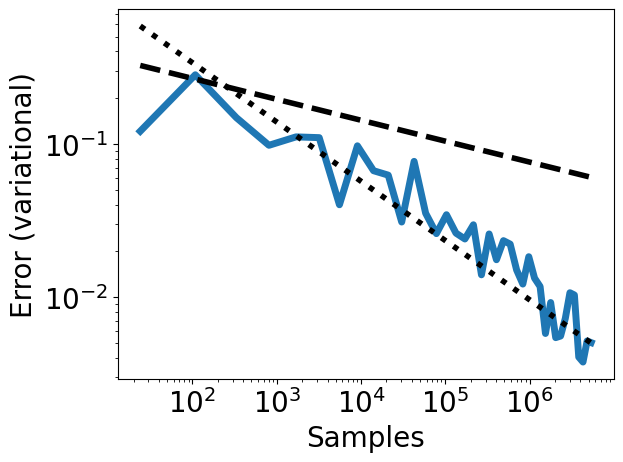}
\!
     \includegraphics[height=3.3cm]{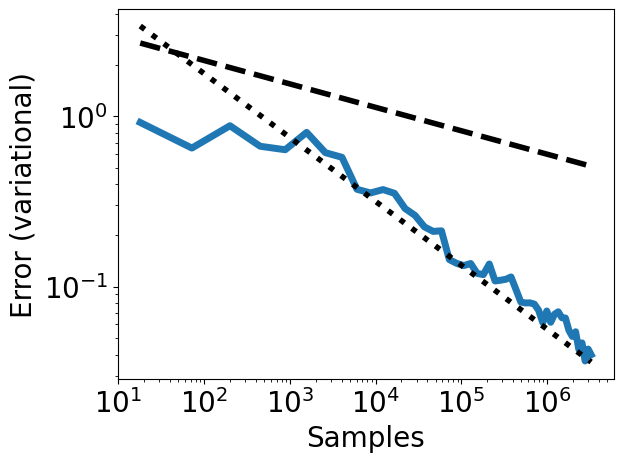}
     \end{tabular}
     \caption{
Left:   $\epsilon=0.3$, $a=1.2$, $b=-0.6$, $d=1$. The theoretical rate of convergence  is $-0.35$ (compared to the old rate of $-0.18$); a linear fit shows OS is converging with rate $-0.31$. Middle: $\epsilon=0.3$, $a=1.7$, $b=-0.6$, $d=2$ with rates $\texttt{Theoretical} =-0.39$; $\texttt{Old theoretical}= -0.14$ ; $\texttt{OS} =  -0.37$. Right:  $\epsilon=0.4$, $a=1.5$, $b=-0.55$, $d=5$ with rates $\texttt{Theoretical} = -0.38$; $\texttt{Old theoretical}= -0.14$ ; $\texttt{OS} = -0.41$. \label{fig:rates_1}}

\end{figure}

\section{Compressed online Sinkhorn}\label{sec:COS}

\subsection{Complexity of online Sinkhorn}\label{subsec:complexityOS}
The online Sinkhorn algorithm's prolonged runtime is due to the evaluation of $(u_t,v_t)$ in \eqref{eq:Sinkhorn_weights_representation} on an increasing amount of newly drawn data. 
In Algorithm \ref{alg:online_sinkhorn}, due to the calculation of the cost matrix $C(x_i,y_j)$ for ${i = (n_t,n_{t+1}], j=(0,n_t]}$ in \eqref{eq:Sinkhorn_weights_representation}, the complexity is $O\pa{n_t b_t}=O\pa{t^{4a+2}}$ and the memory cost for $\pa{x_i,y_i}_{i=(0,n_t]}$ is $O\pa{n_t}=O\pa{t^{2a+1}}$. The complexity and memory cost increase polynomially with each iteration; therefore, the Algorithm \ref{alg:online_sinkhorn} will require a computing unit with higher capacity as the iteration progresses.

\subsection{Compressed online Sinkhorn}

To reduce the complexity and memory costs, we propose the \emph{compressed online Sinkhorn} algorithm.
To explain our derivation, we focus on obtaining a compressed version of $u_t$; the function $v_t$ can be treated in a similar manner.
From \eqref{eq:Sinkhorn_weights_representation},  we have $u_t(x)=\int \frac{K_x\pa{y}}{\phi\pa{y}}\mathrm{d}\mu\pa{y}$ for $\mu \eqdef \sum_{i=1}^{n_t} \exp\pa{\frac{q_{i,t}}{\epsilon}}\phi\pa{y_{i}} \delta_{y_{i}}$, $K_x(y)\eqdef \exp(-C(x,y)/\epsilon)$ and any positive function $\phi$. The idea of our compression method is to exploit measure compression techniques to replace $\mu$  with a measure $\hat \mu $ made up of $m\ll n_t$ Diracs. We then approximate $u_t$ with $\hat u_t(x) = \int \frac{K_x\pa{y}}{\phi\pa{y}}\mathrm{d}\Hat\mu(y)$.
Let $\phi=v_t=\exp\pa{-\frac{g_t}{\epsilon}}$, then $\exp\pa{\frac{q_{i,t}}{\epsilon}}\phi\pa{y_i}\le 1$ by the update \ref{Sinkhorn:step3} Algorithm \ref{alg:online_sinkhorn}. The weights for $\mu$ are hence bounded and it is easier to compress $\mu$.

To ensure that $\hat u_t\approx u_t$, we will enforce that the measure $\hat \mu$ satisfies
\begin{equation}\label{eq:moment}
    \int P_k\pa{y} \mathrm{d}\mu\pa{y} = \int P_k\pa{y} \mathrm{d}\Hat\mu\pa{y}
\end{equation}
for some appropriate set of functions $\enscond{P_k}{k\in\Omega}$, where $\Omega$ is some set of cardinality $m_t$.
Let us introduce a couple of ways to choose such functions.

\paragraph{Example: Gaussian quadrature} In dimension 1 (covered here mostly for pedagogical purposes), one can consider the set of polynomials up to degree $2m-1$, denoted by $\mathbb{P}_{2m-1}$ for some  $m\in \NN$.  For well-ordered $y_i$, the constraints~(\ref{eq:moment}) defines the $m$-point Gaussian quadrature $\hat{\mu}=\sum \hat{w}_i \delta_{\hat{y}_i}$, and both new weights $\hat{w}_i$ and new nodes $\hat{y}_i$ are required to achieve this. Numerically, this can be efficiently implemented following the \texttt{OPQ} Matlab package \cite{gautschi2004orthogonal} and the total complexity for the compression is $O\pa{m^{3} +n_t m}$. To define $\hat u_t$,  solve $e^{\hat q_{i,t}/\epsilon} v_t( \hat y_{i})=\hat{w}_i$ for $\hat{q}_{i,t}$ and let $\hat u_t(x)=\sum_{i=1}^m\exp(\hat{q}_{i,t}/\epsilon) K_x(\hat {y}_{i})$.

\paragraph{Example: Fourier moments}
Since we expect that the compressed function $\hat u_t$ satisfies $\hat u_t \approx u_t$, one approach is to ensure that the Fourier moments of $\hat u_t$ and $u_t$ match on some set of frequencies. Observe that the Fourier transform of $u_t(x) = \int \frac{K_x}{v_t}\pa{y} \mathrm{d}\mu(y)$ can be written as
\begin{equation}
    \int \exp\pa{-ikx}  \int \frac{K_x}{v_t}\pa{y} \mathrm{d} \mu\pa{y} \mathrm{d}x = \int \frac{\Hat K_y(k)}{v_t(y)} \mathrm{d}\mu\pa{y}, \qwhereq \Hat K_y(k)\eqdef \int K_x (y)\exp\pa{-ikx}\mathrm{d}x.
\end{equation}
We therefore let $P_k(y)=\frac{\hat K_y(k)}{v_t\pa{y}} $ in \eqref{eq:moment}, for suitably chosen frequencies $k$ (see Appendix \ref{subsec:Fourier}).

In contrast to GQ, we retain the current nodes $y_i$ and seek new weights $\hat w_i$ such that $\hat\mu=\sum_{i=1}^m \hat{w}_i\delta_{y_i}$ satisfies  $\int \frac{\hat{K}_y(k)}{v_t(y)}\mathrm{d}\pa{\mu-\Hat\mu}\pa{y}=0$. Let  $b_t = \pa{\int P_k(y) d\mu(y)}$ and $M_t = \pa{ P_k(y_i)}_{k\in\Omega, i\in [n_t]}$; equivalently, we seek $\hat w\ge 0$ such that $M_t \hat w=b_t$. We know from the Caratheodory theorem that there exists a solution with only $m+1$ positive entries in $\hat w$, and the resulting compressed measure $\hat{\mu}=\sum_{\hat{w}_i>0} \hat{w}_i \delta_{\hat{y}_i}$. Algorithms to find the Caratheodory solution include \cite{hayakawa2022positively} and \cite{tl}.  We may also find $\hat w\ge 0$ to minimise $\|M_t \hat w-b_t\|^2$ using algorithms such as \cite{2020SciPy-NMeth} and \cite{scikit-learn}, which in practice is faster but leads to a sub-optimal support: in our experiments, these methods contain up to $10m$ points. Theoretically, as long as the size of the support is  $O(m)$, the same behaviour convergence guarantees hold, so although the non-negative least squares solvers do not provide this guarantee, we find it to be efficient and simple to implement in practice.
The complexity of calculating the matrix $M_t$ is $O\pa{dmn_t}$, and for solving the linear system is $O\pa{m^3}$ or $O\pa{m^3\log n_t}$ depending on the solution method. Therefore, the total complexity for the compression method is $O\pa{m^3+dmn_t}$ or $O\pa{m^3\log n+dmn_t}$.
In practice, we found the non-negative least squares solvers to be faster and to give good accuracy. 
The compressed online Sinkhorn algorithm is summarised in  Algorithm \ref{alg:compressed_online_sinkhorn}. The main difference is the Steps \ref{step:9_in_cos} and \ref{step:10_in_cos}, where the further compression steps are taken after the online Sinkhorn updates in Algorithm \ref{alg:online_sinkhorn}. This reduces the complexity of the evaluation of $(f_t,g_t)$ in the next iteration.

\begin{algorithm}[!ht]
    \DontPrintSemicolon
    \caption{Compressed online Sinkhorn}
    \textbf{Input:} Distributions $\al$ and $\be$, learning rates $\pa{\eta_t}_t=\pa{\pa{t+1}^{b}}_t$, batch size $\pa{b_t}_t=\pa{\pa{t+1}^{2a}}_t$, such that $-1<b<-\frac12$ and $a-b>1$\\
    \textbf{Initialisation:} $n_1=m_1 = b_1, \Hat p_{i,1} = \Hat q_{i,1} =0$, $\Hat x_{i,1}\iid \al, \Hat y_{i,1}\iid \be$ for $i \in (0,m_1]$\\
     \For{$t = 1,\cdots, T-1$}{
     \begin{enumerate}
         \item Sample $(x_{i})_{(n_t,n_{t+1}]} \iid \al$, $(y_{i})_{(n_t,n_{t+1}]} \iid \be$, where $n_{t+1}=n_t+b_t$. Let $\pa{\Hat x_{i,t}}_{(m_t,m_t+b_t]}=(x_{i})_{(n_t,n_{t+1}]}$, $\pa{\Hat y_{i,t}}_{(m_t,m_t+b_t]}=(y_{i})_{(n_t,n_{t+1}]}$.\label{step:1_in_cos}
         \item Evaluate $(\Hat{g}_{t}(y_{i}))_{i = (n_t,n_{t+1}]}$, where $\Hat g_{t} = -\epsilon\log\sum_{i=1}^{m_t} \exp{( \frac{\Hat p_{i,t} - C(\Hat x_{i,t},\cdot)}{\epsilon})}$.\label{step:2_in_cos}
         \item $q_{(m_t,m_t+b_t],t+1} \leftarrow \epsilon\log\pa{\frac{\eta_t}{b_t}} + \pa{\Hat{g}_{t}(y_{i})}_{(n_t,n_{t+1}]}$. 
         \item $q_{(0,m_t],t+1} \leftarrow \Hat q_{(0,m_t],t} + \epsilon\log\pa{1-\eta_t}$.
         \item Evaluate $(f_{t+1}(x_{i}))_{i =(n_t,n_{t+1}]}$, where $f_{t+1} = -\epsilon\log\pa{\sum_{i=1}^{m_t+b_t} \exp{\pa{ \frac{q_{i,t+1} - C(\cdot, \Hat y_{i,t})}{\epsilon}}}}$
         \item $p_{(m_t,m_t+b_t],t+1} \leftarrow \epsilon\log\pa{\frac{\eta_t}{b_t}} + (f_{t+1}(x_{i}))_{(n_t,n_{t+1}]}$.
         \item $p_{(0,m_t],t+1} \leftarrow \Hat p_{(0,m_t],t} +\epsilon \log\pa{1-\eta_t}$ .\label{step:7_in_cos}
             \item Compression: find  $\hat u_{t+1}=e^{-\hat f_{t+1}}$ to approximate $u_{t+1}=e^{-f_{t+1}/\epsilon}$ such that \eqref{eq:moment}  holds (for $\mu,\hat \mu$ corresponding to $u_t$, $\hat u_t$) with $m_t$ points. Define $\hat{y}_{i,t+1},\hat{q}_{i,t+1}$.     so that $\hat{f}_{t+1}=-\epsilon \log \sum_{i=1}^{m_t+1} \exp((\hat{q}_{i,t+1}-C(\hat{y}_{i,t+1},\cdot))/\epsilon)$,  \label{step:9_in_cos}
         
         \item Similarly update $\hat{y}_{i,t+1}$ and $\hat{p}_{i,t+1}$.        
         
      \label{step:10_in_cos}
     \end{enumerate}
     }
     \textbf{Returns:} $\Hat f_{T} = -\epsilon\log\sum_{i=1}^{m_T} \exp{( \frac{\Hat q_{i,T} - C(\cdot,\Hat y_{i,T})}{\epsilon})},\Hat g_{T} = -\epsilon\log\sum_{i=1}^{m_t} \exp{( \frac{\Hat p_{i,T} - C(\Hat x_{i,T},\cdot)}{\epsilon})}$
     \label{alg:compressed_online_sinkhorn}
\end{algorithm}

\subsection{Complexity analysis for compressed online Sinkhorn}

In this section, we analyse the numerical gains of our compression technique over online Sinkhorn. 

\paragraph{The key assumptions}
Our complexity analysis is based on the following assumptions for the compression in Steps~\ref{step:9_in_cos} and~ \ref{step:10_in_cos}  in Algorithm \ref{alg:compressed_online_sinkhorn}.
\begin{ass}\label{assumption4}
   At step $t$, approximate $u_t=e^{-f_t/\epsilon}$ by $\hat u_t=e^{-\hat f_t/\epsilon}$(and similarly $v_t$ by $\hat v_t$) with compression size $m_t$.   For a given $\zeta>0$ depending on the compression method, assume that
    % \begin{align*}
    %     \underset{x}{\sup}\abs{\frac{\int J_x(y) \mathrm{d}(\mu-\Hat \mu)(y)}{\int J_x(y) \mathrm{d}\mu (y)}}=O\pa{m_t^{-\zeta}}\qandq
    %     \underset{y}{\sup}\abs{\frac{\int J_y(x) \mathrm{d}(\nu-\Hat\nu)(x)}{\int J_y(x) \mathrm{d}\nu (x)}}=O\pa{m_t^{-\zeta}},
    % \end{align*}
      \begin{align*}
        \underset{x}{\sup}\abs{f_t(x)-\hat {f}_t(x)}=O\pa{m_t^{-\zeta}}\qandq
        \underset{y}{\sup}\abs{g_t(y)-\hat{g}_t(y)}=O\pa{m_t^{-\zeta}}.
    \end{align*}
\end{ass}

\begin{ass}\label{assumption5}
    The compression size is $m_t=t^{\lambda}$ for $\lambda=\frac{a-b}{\zeta}$. % $0<\lambda<2a$
\end{ass}

Let us return to our examples from Section \ref{subsec:complexityOS} and show the assumptions are satisfied for  Gaussian quadrature and Fourier moments compression. The details can be found in Appendix~\ref{subsec:compression_err}.

\paragraph{Gaussian quadrature}
Note that $K_x$ and $v_t$ are smooth (both have the same regularity as $e^{-C(x,y)/\epsilon}$; see \eqref{eq:Sinkhorn_weights_representation}). By \cite[Corollary to Theorem 1.48]{gautschi2004orthogonal}, the compression error by Gaussian quadrature is 
$$
u_t(x)-\hat u_t(x)=\int \frac{K_x}{v_t}(y) \mathrm{d}\mu(y)-\int \frac{K_x}{v_t} (y) \mathrm{d}\Hat\mu (y) =O\pa{ \frac{1}{\epsilon^m m!}}.$$
As $f_t(x)$ is a Lipschitz function of $u_t(x)$, we also have $f_t(x)-\hat f_t(x)=O\pa{ \frac{1}{\epsilon^m m!}}$.
We see Assumption~\ref{assumption4} holds for any $\zeta>0$, as  $\pa{m!}^{-1}={O}\pa{m^{-\zeta}}$. 
% {\color{green} doesn't this bound depend on regularity of $K_x(y)/v_t(y)$? All we know about $v_t$ is that it is Lipschitz but the GQ error bound needs higher differentiability?}

\paragraph{Fourier moments}
We show in Appendix \ref{subsec:compression_err} that, for $C(x,y)=\norm{x-y}^2$,
\begin{equation}
 u_t(x)-\hat u_t(x)=   \int \frac{K_x}{v_t}(y) d\mu(y)- \int \frac{K_x}{v_t}(y) d\hat \mu(y) = \int \frac{\phi_x\pa{z}}{v_t\pa{y}} \mathrm{d} z,
\end{equation}
where
\begin{equation}
    \phi_x\pa{k} =\exp\pa{i\,z\,k} \sqrt{\epsilon\pi}\exp\pa{-\frac{\epsilon k^2}{4}} \int \exp\pa{-i\,k\,y}\mathrm{d}\pa{\mu-\Hat\mu}\pa{y}.
\end{equation}
By choice of $P_k$ and \eqref{eq:moment}, $\phi_x(k)=0$ for  $k\in\Omega$ and the compression error via Gaussian QMC sampling \cite{kuo2016practical} is  
$ O\pa{\frac{\abs{\log m}^d}{m}}.
$
For any $\zeta<1$,  $\frac{\abs{\log m}^d}{m}={O}\pa{m^{-\zeta}}$ and Assumption \ref{assumption4} holds.

We now present the convergence theorem for Algorithm \ref{alg:compressed_online_sinkhorn}. 
\begin{thm}\label{thm:COSerror}   Let $f^*$ and $g^*$ denote the optimal potentials. Let  $\Hat{f}_t$ and $\Hat{g}_t$ be the output of Algorithm \ref{alg:compressed_online_sinkhorn} after $t$ iterations. 
    Under Assumptions \ref{assumption1} to \ref{assumption5}, $\eta_t =  t^{b}$ for $-1<b<-\frac{1}{2}$, and $b_t = t^{2a}$ with $a-b>1$, and the compression size $m_t=t^{(a-b)/\zeta}$. Then
    \begin{equation}
        \Hat\delta_N \lesssim  \exp\pa{- c N^\frac{b+1}{2a+1}}+N^{-\frac{a}{2a+1}} = {O}\pa{N^{-\frac{a}{2a+1}}},
    \end{equation}
    where $\Hat\delta_N = \norm{\pa{ \Hat f_{t\pa{N}} - f^*}/\epsilon}_{var} + \norm{ \pa{\Hat g_{t\pa{N}} - g^*}/\epsilon}_{var}$, ${t\pa{N}}$ is the first iteration number for which $\sum_{i=1}^t b_i >N$, and $c$ is a positive constant.
\end{thm}
\paragraph{Sketch of proof}
The proof follows the proof of Theorem \ref{thm:OSerror} with the extra compression errors. Further to the definitions of $U_t$ and $V_t$, define the following terms for $\Hat f_t$ and $\Hat g_t$:
\begin{align}
    &\Hat U_t \eqdef \pa{ \Hat f_t - f^*}/\epsilon, & \Hat V_t &\eqdef \pa{\Hat g_t - g^*}/\epsilon,\\
        &\Hat\zeta_{\Hat\be_t} \eqdef \pa{T_{\be}(\Hat g_t) - T_{\Hat\be_t}(\Hat g_t)}/\epsilon, & \Hat \iota_{\Hat\al_t}&\eqdef \pa{T_\al(f_{t+1})- T_{\Hat\al_t}(f_{t+1})}/\epsilon.
\end{align}

% Further define the following terms regarding $\Hat f_t$ and $\Hat g_t$,
%     \begin{align*}
%         &\Hat U_t \eqdef \pa{ \Hat f_t - f^*}/\epsilon, & \Hat V_t &\eqdef \pa{\Hat g_t - g^*}/\epsilon,\\
%         &\Hat U_t^T \eqdef \pa{T_{\al}(\Hat f_t) - T_{\al}(f^*)}/\epsilon, & \Hat V_t^T &\eqdef \pa{T_{\be}(\Hat g_t) - T_{\be}(g^*)}/\epsilon,\\
%         &\Hat\zeta_{\Hat\be_t} \eqdef \pa{T_{\be}(\Hat g_t) - T_{\Hat\be_t}(\Hat g_t)}/\epsilon, & \Hat \iota_{\Hat\al_t}&\eqdef \pa{T_\al(f_{t+1})- T_{\Hat\al_t}(f_{t+1})}/\epsilon.
%     \end{align*}

    Then we have the following relation
    \begin{align*}
        \Hat U_{t+1} &= \pa{ \Hat f_{t+1} - f_{t+1}}/\epsilon + \pa{ f_{t+1} - f^*}/\epsilon \eqdef \text{err}_{f_{t+1}} + U_{t+1},\\
        \Hat V_{t+1} &= \pa{ \Hat g_{t+1} - g_{t+1}}/\epsilon + \pa{ g_{t+1} - g^*}/\epsilon \eqdef \text{err}_{g_{t+1}} + V_{t+1},
    \end{align*}
    where $\text{err}_{f_{t+1}} = \pa{ \Hat f_{t+1} - f_{t+1}}/\epsilon$ and $\text{err}_{g_{t+1}} = \pa{ \Hat g_{t+1} - g_{t+1}}/\epsilon$. Notice that $\norm{\text{err}_{f_{t+1}}}_{var}=O\pa{t^{-a+b}}$ under Assumptions \ref{assumption4} and \ref{assumption5}.

    Define $\Hat e_t \eqdef \|\Hat U_t\|_{var} + \|\Hat V_t\|_{var}$, then for $t$ large enough
    \begin{equation}
            \Hat e_{t+1} \leq \pa{ 1-\eta_t + \eta_t \kappa } \Hat e_{t} + \eta_t \pa{\|\Hat \zeta_{\Hat\be_t}\|_{var}+ \|\Hat\iota_{\Hat\al_t}\|_{var}} + \norm{\text{err}_{f_{t+1}}}_{var} + \norm{\text{err}_{g_{t+1}}}_{var}.
    \end{equation}

    Similarly to the proof of Theorem \ref{thm:OSerror}, we can show that 
    $$\mathbb{E}\norm{\zeta_{\Hat\be_t}}_{var}\lesssim \frac{A_1(f^*,g^*,\epsilon)}{\sqrt{b_t}} \qandq \mathbb{E}\norm{\iota_{\Hat\al_t}}_{var}\lesssim \frac{A_2(f^*,g^*,\epsilon)}{\sqrt{b_t}},$$
    where $A_1(f^*,g^*,\epsilon),A_2(f^*,g^*,\epsilon)$ are constants depending on $f^*,g^*,\epsilon$.

    Taking expectations and applying the Gronwall lemma \cite[Lemma 5.1]{gronwell-notes} to the recurrence relation, we have 
    \begin{equation}
        \mathbb{E} \Hat e_{t+1} \lesssim  \pa{\mathbb{E}e_1 + \frac{1}{a-b-1}} \exp\pa{\frac{\kappa-1}{b+1} t^{b+1}}+ \frac{1}{\pa{1-\kappa}} t^{-a},
    \end{equation}
\hfill\qedsymbol{}

To reach the best asymptotic behaviour of the convergence rate $N^{-\frac{a}{2a+1}}$ in Theorem \ref{thm:COSerror}, we want $a\to\infty$. In this limit however, the transient behaviour $\exp(-c N^{\frac{b+1}{2a+1}})$ becomes poorer and, in practice, a moderate value of $a$ must be chosen. 

\begin{prop}\label{prop:compression_complexity}
  
    Under Assumptions \ref{assumption4} and \ref{assumption5}, with a further assumption that compressing a measure from $n$ atoms to $m$ has complexity $C\pa{n,m}=O\pa{m^{3}+nm}$, the computational complexity of reaching accuracy $\mathbb{E}e_t \leq \delta$
    for Algorithm \ref{alg:compressed_online_sinkhorn} is
    \begin{equation}\label{eq:ratios}
        \Hat \Cc = O\pa{\delta^{-\pa{2+\frac{a-b}{a \zeta}+\frac{1}{a}}}},\quad \forall \zeta\geq\frac{a-b}{a}\qandq \Hat \Cc = O\pa{\delta^{-\pa{\frac{3\pa{a-b}}{a\zeta}+\frac1a}}},\quad \forall \zeta< \frac{a-b}{a}
    \end{equation}
    
    The complexity of reaching the same accuracy $\delta$ for online Sinkhorn is
    $\Cc = O\pa{\delta^{-\pa{4+\frac2a}}}.$
\end{prop}

The ratio of the complexities for the compressed and original online Sinkhorn algorithm is
\begin{equation}
    \frac{\Hat \Cc}{\Cc}  = \Oo\pa{\delta^{2+\frac1a-\frac{a-b}{a \zeta}}},\quad \forall \zeta\geq\frac{a-b}{a} \qandq  \frac{\Hat \Cc}{\Cc} =  
    \Oo\pa{\delta^{4+\frac1a-\frac{3\pa{a-b}}{a\zeta}}},\quad \forall \zeta< \frac{a-b}{a}. 
\end{equation}
% \begin{equation}
%     \frac{\Hat \Cc}{\Cc}  = \begin{cases} \Oo\pa{\delta^{2+\frac1a-\frac{a-b}{a \zeta}}},  \text{ for } \zeta\geq\frac{a-b}{a}, \\ 
%     \Oo\pa{\delta^{4+\frac1a-\frac{3\pa{a-b}}{a\zeta}}}, & \text{ for } \zeta< \frac{a-b}{a}. \end{cases}
% \end{equation}
The exponent is positive when $\zeta > \frac{3\pa{a-b}}{4a+1}$, indicating that the compressed Online Sinkhorn is more efficient. The larger this exponent, the more improvement we can see in the asymptotical convergence of the compressed online Sinkhorn compared to the online Sinkhorn.

\section{Numerical experiments}\label{sec:COS_num}

In Figure \ref{fig:compressedOS}, we compare online Sinkhorn (OS) with compressed online Sinkhorn (COS) with Fourier compression in $1$D, $2$D and $5$D (choosing $\zeta=0.95$ in $1$D and $0.9$ in $2$ and $5$D), and with Gauss quadrature (GQ) in $1$D (choosing $\zeta=2$).  Recall that the learning rate is $\eta_t=t^b$ and the batch size at step $t$ is $b_t=t^{2a}$, and we chose different sets of parameters for the experiments. The compression methods are applied once the total sample size reaches $1000$. In (a), we use the same Gaussian distribution as in Section~\ref{sec:OS_numerical}. For (b) and (c),  we used a Gaussian Mixture Model (GMM) as described in Appendix~\ref{sec:gmm}. 
We display the errors of the potential functions $\norm{f_t-f_{t-1}}_{var} + \norm{g_t-g_{t-1}}_{var}$ (for a variation distance based on a finite set of points) and the relative errors of the objective functions $F\pa{f_t,g_t} - F\pa{f^*,g^*}$. For computing relative errors, the exact value is explicitly calculated for the Gaussian example (a) following~\cite{janati2020entropic}. For the Gaussian Mixture Models (b) and (c),  the reference value is taken from online Sinkhorn with approximately $N=10^5$ samples.
For the parameters choice in Figure \ref{fig:compressedOS}, in the $1$D experiments, the complexities of GQ and Fourier COS are $\Hat\Cc=O\pa{\delta^{-\frac{101}{30}}}$ and $\Hat\Cc=O\pa{\delta^{-\frac{290}{57}}}$ respectively, with the complexity of OS being $\Cc=O\pa{\delta^{-\frac{16}{3}}}$. In $2$ and $5$D experiments, the complexities for the Fourier COS are $\Hat\Cc=O\pa{\delta^{-\frac{35}{6}}}$, where the complexity for OS is $\Cc=O\pa{\delta^{-\frac{17}{3}}}$. As shown in Figure \ref{fig:compressedOS}, even when the choice of $\zeta$ is outside the optimal range in $2$ and $5$D, we still observe a better running time.

The numerical simulations were implemented in Python and run using the NVIDIA Tesla K80 GPU on Colab. To reduce the memory cost, we employed the KeOps package \cite{JMLR:v22:20-275} calculating the $C$-transforms in Algorithm \ref{alg:online_sinkhorn}. Execution times are shown and demonstrate the advantage of the compression method. Here the compression is computed using Scipy's \texttt{NNLS}.% the execution times using \cite{tl}, which finds the measure with optimal support, are much longer at (a) 583s, (b), 392s, and (c)  140s.

\begin{figure}
    \centering
    \begin{tabular}{ccc}
    (a)&(b)&(c)\\
    \includegraphics[width=0.3\linewidth]{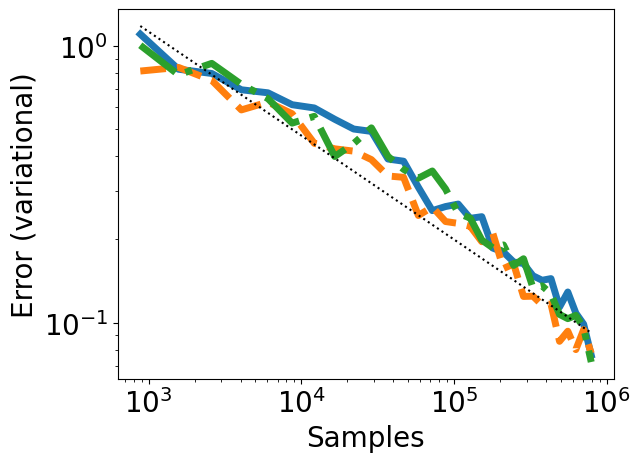}
         &   \includegraphics[width=0.3\linewidth]{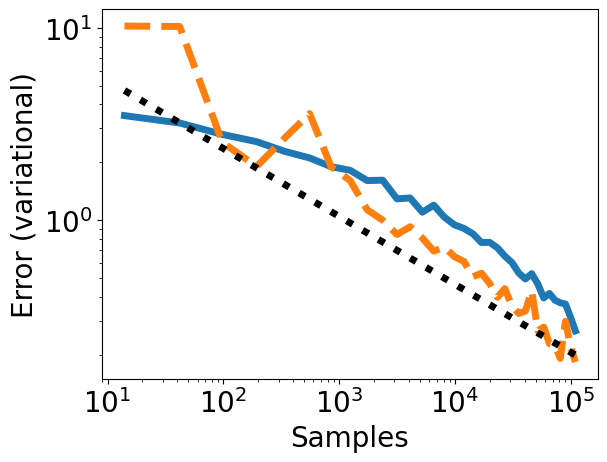}
         & \includegraphics[width=0.3\linewidth]{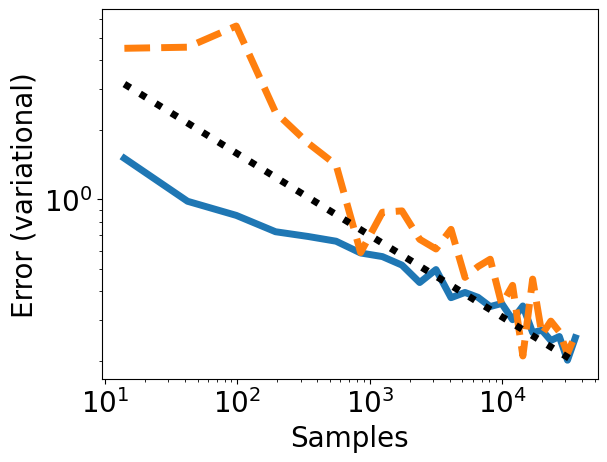}\\
         \includegraphics[width=0.3\linewidth]{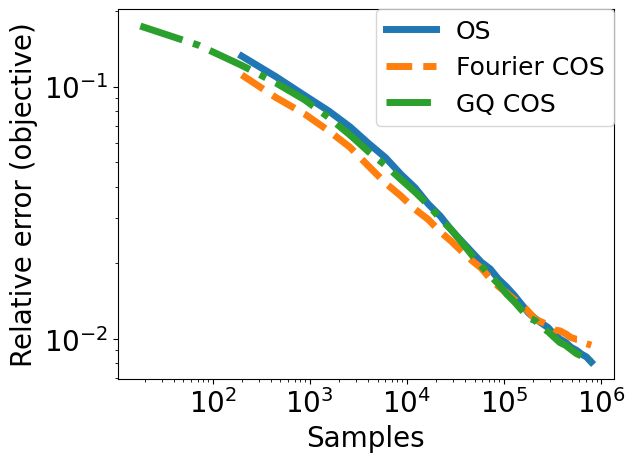}
         &\includegraphics[width=0.3\linewidth]{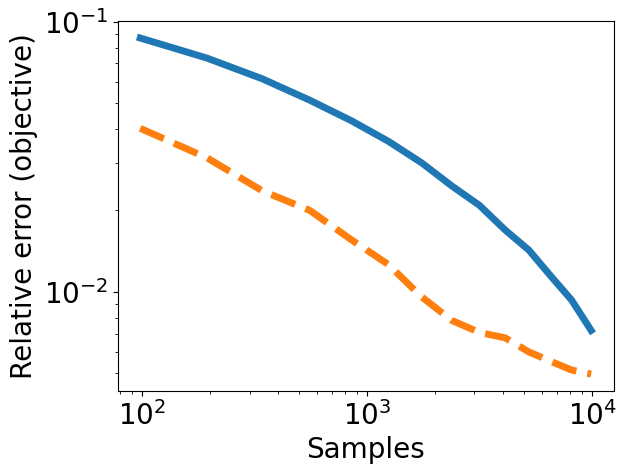}
         &\includegraphics[width=0.3\linewidth]{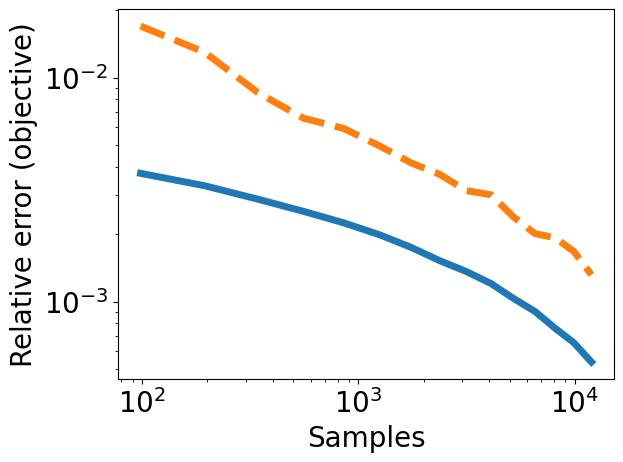}
    \end{tabular}
    \caption{(a)  $\epsilon = 0.4$, $a= 1.5$, $ b= -0.6$, $d=1$. Execution times: OS is $367s$, Fourier COS  is $130s$ and GQ COS  is $151s$. (b) $\epsilon = 0.5$, $a= 1.2$, $b= -0.6$, $d=2$.
Execution times: OS is $433s$  and Fourier COS is $207s$. (c) $\epsilon = 0.5$,
$a= 1.2$ ; $b= -0.6$,  $d=5$.
Execution times: OS is $386s$ and Fourier COS is $92s$.\label{fig:compressedOS}}
\smallskip
\end{figure}

\section{Conclusion and future work}
% \todo{write me}
% \begin{enumerate}
%     \item Non-negative linear square works well in practice, other compression methods;
%     \item the choice of the parameters $a,b$ so that our algorithm benefits from the fast convergence of Sinkhorn, and the asymptotical convergence rate in Theorem \ref{thm:OSerror}
%     \item choice of $\zeta$ in numerical experiments
%     \item when $C$ is not quadratic
% \end{enumerate}

In this paper, we revisited the online Sinkhorn algorithm and provided an updated convergence analysis. We also proposed to combine this algorithm with measure compression techniques and demonstrated both theoretical and practical gains. We focused on the use of Fourier moments as a compression technique, and this worked particularly well with quadratic costs due to the fast tail behaviour of $\exp(-x^2)$. However, Fourier compression generally performs less well for non-quadratic costs (such as $C(x,y) = \norm{x-y}_1$) due to the slower frequency decay of $\exp(-C)$. We leave investigations into compression of such losses with slow tail behaviour as future work. Another direction of future work is to further investigate the non-asymptotic convergence behaviour: In Theorem \ref{thm:OSerror}, there is a linear convergence term which dominates in the non-asymptotic regime and it could be beneficial to adaptively choose the parameters $a,b$ to exploit this fast initial behaviour. 

\bibliography{references}
\bibliographystyle{apalike}

\appendix
\section{Appendix}\label{sec:appendix}
\subsection{Gaussian mixture model}\label{sec:gmm}

For experiments (b) and (c) in Figure \ref{fig:compressedOS}, we set up test distributions $\alpha,\beta$ that sample from $\mathcal{N}(\mu_i,\Sigma_i)$ for $i=1,2$ with equal probability.
For (b), the entries of $\mu_1$ are iid samples from $\mathcal{N}(0,100)$ for $\alpha$,  and from  $\mathcal{N}(0,25)$ for $\beta$; we take $\mu_2=-\mu_1$ in both cases. 
The covariance matrices $\Sigma_i=c_i\,Q Q^T$ for $Q$ with iid standard Gaussian entries using $c_1=3$ and $c_2=4$. 
For (c), the same set-up is used with $c_1=1$ and $c_2=0.4$. 

\subsection{Useful lemmas}\label{subsec:lemmas}
Define the operator $T_{\mu}$ to be \cite{mensch2020online}
\begin{equation}
    T_{\mu}\pa{h} = -\log\int_{y\in\Xx} \exp\pa{h(y)-C\pa{\cdot,y}}\mathrm{d}\mu(y),
\end{equation}
then the updates with respect to the potentials $\pa{f_t,g_t}$ are
\begin{equation}
f_{t+1} = T_{\be}(g_{t}) \qandq g_{t+1} = T_{\al}(f_{t}).
\end{equation}

Under Assumption \ref{assumption1}, a soft $C$-transform is always Lipschitz, and the following result can be found in \cite[Proposition 15]{vialard2019elementary}. % \todo{I don't see where you defined $L$, please also check rest of paper that it was defined before usage}
\begin{lem}\label{lemma_Lipschitz}
     Under Assumption \ref{assumption1}, a soft $C$-transform $f = T_{\mu}(g)$ with a probability measure $\mu$ is $\epsilon L$-Lipschitz, where $L$ is the Lipschitz constant of $C$ defined in Assumption \ref{assumption1}.
\end{lem}

We also show that the error $\underset{x\in\Xx}{\sup}\abs{f_t-f^*},\underset{x\in\Xx}{\sup}\abs{g_t-g^*}$ in the online Sinkhorn is uniformly bounded.
\begin{lem}\label{lemma_bound}
    Suppose for some $t>0$, $\max\ens{\underset{x\in\Xx}{\sup}\abs{f_t-f^*}<\delta,\underset{y\in\Xx}{\sup} \abs{g_t-g^*}}<\delta$, where $\pa{f^*,g^*}$ is the pair of optimal potentials. Let $\pa{f_{t+1},g_{t+1}}$ be the pair of potentials in the following updates
    \begin{align}
        \exp\pa{-f_{t+1}(x)/\epsilon} &= \pa{1-\eta_t} \exp\pa{-f_t(x)/\epsilon} + \eta_t \int \exp\pa{\pa{g_t(y) - C(x,y)}/\epsilon}\mathrm{d}\Hat\be_t(y),\label{eq1:updates_f}\\
        \exp\pa{-g_{t+1}(x)/\epsilon} &= \pa{1-\eta_t} \exp\pa{-g_t(x)/\epsilon} + \eta_t \int \exp\pa{\pa{f_{t+1}(y) - C(x,y)}/\epsilon} \mathrm{d}\Hat\al_t(x),\label{eq2:updates_g}
    \end{align}
    where $\Hat\al_t,\Hat\be_t$ are two probability measures.
    Then $\underset{x\in\Xx}{\sup}\abs{f_{t+1}-f^*}<\delta$.
\end{lem}
\begin{proof}
    Multiply by $\exp{\pa{f^*/\epsilon}}$ on both sides of the \eqref{eq1:updates_f},
    \begin{align*}
    \exp(f^*/\epsilon-f_{t+1}/\epsilon) &= (1-\eta_t) \exp(f^*/\epsilon-f_t/\epsilon)  + \eta_t \int \exp((f^* + g_t - C)/\epsilon) \mathrm{d}\Hat\be_t,\\
    &= (1-\eta_t) \exp(f^*/\epsilon-f_t/\epsilon)  + \eta_t \int \exp((f^* + g^*-g^*+ g_t - C)/\epsilon) \mathrm{d}\Hat\be_t,\\
    &= (1-\eta_t) \exp(f^*/\epsilon-f_t/\epsilon)  + \eta_t \int \exp(( g_t - g^*)/\epsilon) \mathrm{d}\Hat\be_t\\
    &<(1-\eta_t)\exp{\pa{\delta/\epsilon}} + \eta_t\int \exp{\pa{\delta/\epsilon}} \mathrm{d}\Hat\be_t=\exp{\pa{\delta/\epsilon}}.
    \end{align*}
    
    Take logs on both sides to get $\underset{x\in\Xx}{\sup}\abs{f^*-f_{t+1}}<\delta$. A similar argument gives the lower bound and we see that $\underset{x\in\Xx}{\sup}\abs{f^*-f_{t+1}}<\delta$.
\end{proof}

As an example of \cite[Chapter 19]{van2000asymptotic}, we have the following results for function $\phi(x) = \int\exp\pa{\pa{f(x)+g(y) - C(x,y)}/\epsilon}\mathrm{d}\be(y)$.
\begin{lem}\label{lem:large_number}
    Under Assumption 1, let $y_1,\cdots,y_{n}$ be i.i.d. random samples from a probability distribution $\be$ on $\Xx$, and define for Lipschitz functions $f$ and $g$,
    \begin{align}\label{empirical_phi}
        &\phi_i(x) = \exp\pa{\pa{f(x)+g(y_i)-C(x,y_i)}/\epsilon},\nonumber \\
        &\phi(x) = \int\exp\pa{\pa{f(x)+g(y) - C(x,y)}/\epsilon}\mathrm{d}\be(y).
    \end{align}
    Suppose that $\norm{\exp\pa{\pa{f(x)+g(y) - C(x,y)}/\epsilon}}_\infty\leq M$ for some $M>0$, where $\norm{\cdot}_{\infty}$ is the supremum norm. Then there exists $A>0$ depending on $f,g,\epsilon$ such that, for all $n>0$,
    \[
    \mathbb{E}\sup_{x\in\Xx} \left| \frac{1}{n}\sum_{i=1}^n\phi_i(x) -\phi(x) \right|\leq \frac{A}{\sqrt{n}}.
    \]
    
    Moreover, $\underset{x\in\Xx}{\sup} \left| \frac{1}{n}\sum_{i=1}^n\phi_i(x) -\phi(x) \right|\xrightarrow{\text{a.s.}} 0$ as $n \to\infty$.
\end{lem}

Make use of Lemma \ref{lem:large_number} to further prove the difference of two $T_{\be}(g)$ and $T_{\Hat\be}(g)$ has an upper bound in $\norm{\cdot}_{var}$.

\begin{lem}\label{lem:error_term_bound}
    Suppose that $g$ is Lipshitcz and $\pa{f^*,g^*}$ is the pair of optimal potentials. Denote \hspace{0.2em}$\zeta_{\Hat\be}\pa{x} \eqdef \frac{\pa{T_{\be}(g) - T_{\Hat\be}(g)}}{\epsilon}\pa{x}$, where $\Hat\be = \frac{1}{n}\sum_{i=1}^{n}\delta_{y_i}$ is an empirical probability measure with $y_i\iid\be$, $n\in\mathbb{Z}_+$, and $M\pa{x}=\int\exp\pa{\pa{f^*\pa{x}+g\pa{y}-C\pa{x,y}}/\epsilon}\mathrm{d}\be\pa{y}$ is bounded from above and below with respect to $x$.
    
    Then, there exists a constant $A_1(f^*,g^*,\epsilon)$ depending on $f^*,g^*,\epsilon$, such that
    \begin{equation}
        \mathbb{E}\|\zeta_{\Hat\be}\|_{var}\lesssim \frac{A_1(f^*,g^*,\epsilon)}{\sqrt{n}}\label{eq:zeta_t_bound}.
    \end{equation}
    
\end{lem}

\begin{proof}
    For all $x \in \Xx$, using that $f^*= T_{\be}(g^*)$,
    \begin{align*}
        \norm{\zeta_{\Hat\be}}_{var} &= \norm{\pa{T_{\be}(g) - T_{\Hat\be}(g)}/\epsilon}_{var} = \norm{\log\pa{\frac{\exp\pa{-T_{\Hat\be}(g)/\epsilon}}{\exp\pa{- T_{\be}(g)/\epsilon}}}}_{var}\\
        &= \norm{\log\pa{1+\frac{\exp\pa{-T_{\Hat\be}(g)/\epsilon}-\exp\pa{- T_{\be}(g)/\epsilon}}{\exp\pa{- T_{\be}(g)/\epsilon}}}}_{var}\\
        &= \norm{\log\pa{1+\frac{ \frac{1}{n}\sum_{i=1}^{n}\phi_i -\phi}{\int\exp\pa{(f^*+g-C)/\epsilon}d \be}}}_{var}\\
        &= \norm{\log\pa{1+\frac{ \frac{1}{n}\sum_{i=1}^{n}\phi_i -\phi}{M}}}_{var},
    \end{align*}
    where $\phi_i(x) = \exp\pa{(f^*(x)+g(y_i)-C(x,y_i))/\epsilon}$ and $\phi(x) = \int\exp\pa{(f^*(x)+g(y) - C(x,y))/\epsilon}\mathrm{d}\be(y)$. Note that for any $x$
    \[
    \frac{ \frac{1}{{n}}\sum_{i=1}^{n}\phi_i(x) -\phi(x)}{M\pa{x}}=\frac{\exp\pa{-T_{\Hat\be}(g)/\epsilon}}{\exp\pa{- T_{\be}(g)/\epsilon}} -1 > -1.
    \]
    
    Let $B\pa{x}=\frac{ \frac{1}{{n}}\sum_{i=1}^{n}\phi_i(x) -\phi(x)}{M\pa{x}}$, then $\|\zeta_{\Hat\be}\|_{var} \leq 2\|\zeta_{\Hat\be}\|_{\infty} = 2\norm{\log\pa{1+B}}_\infty$.
    
    By Lemma \ref{lem:large_number} and the assumption that $n$ is lower bounded, $\frac{1}{n}\sum_{i=1}^n\phi_i(x) -\phi(x)$ is upper bounded with respect to $x$. 
    Hence, there exists $-1<L<U$, such that $L<B\pa{x}<U$ for any $x$. Therefore, for any $x$
    \[
        \log(1+L)<\log\pa{1+B\pa{x}}<\log(1+U),
    \]
    and $\|\zeta_{\Hat\be}\|_{var}< 2\max\{|\log(1+L)|,|\log(1+U)|\}\coloneqq Q$ for all $t$.

    % \begin{align*}
    %     \|\zeta_{\Hat\be_t}\|_{var} &\leq 2\|\zeta_{\Hat\be_t}\|_{\infty} = 2 \sup_{x\in\Xx}\pa{\log\pa{\frac{\exp\pa{-T_{\Hat\be_t}(g_t)/\epsilon}}{\exp\pa{- T_{\be}(g_t)/\epsilon}}}} \\
    %     &=\sup_x 2\log\pa{\frac{\exp\pa{(g_t(x)-c)/\epsilon}}{\exp\pa{({g_t}(x)-c)/\epsilon}}} \\
    %       &=\sup_x 2\log\pa{\frac{\exp\pa{(g_t(x)-g^*+g^*+f^*-c)/\epsilon}}{\exp\pa{({g_t}(x)-g^*+g^*+f^*-c)/\epsilon}}}  \\
    %       &=\sup_x 2\log\pa{\frac{\exp\pa{(g_t(x)-g^*)/\epsilon}}{\exp\pa{({g_t}(x)-g^*)/\epsilon}}} 
    % \end{align*}
    
    % Notice that $n_t\to\infty$ as $t\to\infty$ so $E_{n_t}\to 0$ by \eqref{eq:monday}. 
    
    On the event $\Omega_t = \left\{ B\pa{x}\leq\frac{1}{2} \right\}$, by applying $\log(1+y)\leq -\log(1-y)\leq 2y$ for $0\leq y\leq\frac{1}{2}$
    \[
        \|\zeta_{\Hat\be}\|_{var} \leq 2\max\left\{ \log\pa{1+B\pa{x}}, -\log\pa{1-B\pa{x}}\right\}=-2\log(1-B\pa{x})\leq 4B\pa{x}.
    \]
    
    By the Markov inequality, $\mathbb{P}\left[ B >\frac{1}{2} \right]\leq 2\mathbb{E}(B)$. Split the expectation $\mathbb{E}\left[\|\zeta_{\Hat\be}\|_{var} \right]$ into two parts, by Lemma \ref{lem:large_number} and the fact that $M$ is bounded from above and below with respect to $x$, there exists a constant $A_1(f^*,g^*,\epsilon)$ depending on $f^*,g^*,\epsilon$,
    
    \begin{equation}
        \begin{split}
            \mathbb{E}\|\zeta_{\Hat\be}\|_{var}&=\mathbb{P}\left[ B \leq \frac{1}{2} \right] \mathbb{E}\left[\|\zeta_{\Hat\be}\|_{var} \middle| B\leq\frac{1}{2} \right] + \mathbb{P}\left[ B > \frac{1}{2} \right] \mathbb{E}\left[\|\zeta_{\Hat\be}\|_{var} \middle| B>\frac{1}{2} \right]\\
        &\leq \mathbb{E}(4B + 2 B Q)= (4+2Q)\mathbb{E}B\\
        &\lesssim \frac{A_1(f^*,g^*,\epsilon)}{\sqrt{n}}.
        \end{split}
    \end{equation} 
\end{proof}

The following lemma shows that the error in the variational norm at this step can be bounded using the error in the variational norm from the last step.
\begin{lem}\label{lem:e_{t+1}_and _e_t}
    Given the empirical measures $\Hat\al = \frac{1}{n}\sum_{i=1}^{n}\delta_{x_i}$, $\Hat\be = \frac{1}{n}\sum_{i=1}^{n}\delta_{y_i}$, where $x_i\iid\al$, $y_i\iid\be$ and $n\in\mathbb{Z}_+$.
    Let $f,g$ be functions of\:\:$C$-transform, consider the update in the online Sinkhorn algorithm for $\eta_t$ at step $t$, 
    \begin{align}
        \exp(-\Hat f(x)/\epsilon) &= (1-\eta_t) \exp(-f(x)/\epsilon)  + \eta_t \int \exp((g(y) - C(x,y))/\epsilon) \mathrm{d}\Hat\be(y),\label{eq1:C-transfer_updates}\\
        \exp(-\Hat g(x)/\epsilon) &= (1-\eta_t) \exp(-g(x)/\epsilon)  + \eta_t \int \exp((\Hat f(y) - C(x,y))/\epsilon) \mathrm{d}\Hat\al(x).\label{eq2:C-transfer_updates}
    \end{align}

    Denote $U \eqdef \pa{ f - f^*}/\epsilon$, $V \eqdef \pa{g - g^*}/\epsilon$, $\Hat U\eqdef \pa{ \Hat f - f^*}/\epsilon$, $\Hat V\eqdef \pa{\Hat g - g^*}/\epsilon$, $\zeta_{\Hat\be} \eqdef \pa{T_{\be}(g) - T_{\Hat\be}(g)}/\epsilon$, and $\iota_{\Hat\al}\eqdef \pa{T_\al(\Hat f)- T_{\Hat\al}(\Hat f)}/\epsilon$, then for $t$ large enough,
    \begin{equation}
        \|\Hat U\|_{var} + \|\Hat V\|_{var} \leq \pa{ 1-\eta_t + \eta_t \kappa }\pa{\|U \|_{var} + \|V\|_{var}} +\eta_t\pa{\|\zeta_{\Hat\be}\|_{var} +\|\iota_{\Hat\al}\|_{var}}.
    \end{equation}
\end{lem}
\begin{proof}
    Multiply $\exp\pa{ f^*/\epsilon}$ on both sides on \eqref{eq1:C-transfer_updates} we get
    \begin{align*}
         \exp\pa{ (-\Hat f+f^*)/\epsilon} &= (1-\eta_t) \exp\pa{ (-f+f^*)/\epsilon}  + \eta_t\exp\pa{f^*(x)/\epsilon}\int \exp \pa{(g(y) - C(x,y))/\epsilon}\mathrm{d}\Hat\be(y)\\
         &= (1-\eta_t) \exp\pa{ (-f+f^*)/\epsilon}  + \eta_t \exp{\pa{-T_{\Hat\be}(g)/\epsilon +T_{\be}(g^*)/\epsilon}}\\
         &= (1-\eta_t) \exp\pa{ (-f+f^*)/\epsilon}  + \eta_t \exp{\pa{\pa{T_\be(g)-T_\be(g)-T_{\Hat\be}(g) +T_{\be}(g^*)}/\epsilon}}\\
         &= (1-\eta_t) \exp\pa{ (-f+f^*)/\epsilon}  + \eta_t \exp \pa{-T_{\be}(g)/\epsilon + T_{\be}(g^*) /\epsilon + \zeta_{\Hat\be}},
    \end{align*}
    and similarly, multiply $\exp\pa{ g^*/\epsilon}$ on both sides of \eqref{eq2:C-transfer_updates},
    \begin{equation}
        \exp\pa{ (-\Hat g+g^*)/\epsilon} = (1-\eta_t) \exp\pa{ (-g+g^*)/\epsilon}  + \eta_t \exp\pa{ -T_{\al}(\Hat f)/\epsilon + T_{\al}(f^*)/\epsilon + \iota_{\Hat\al}}
    \end{equation}
    where $(f^*,g^*)$ is the pair of optimal potentials.
    
    Denote $\Hat U^T \eqdef \pa{T_{\al}(\Hat f) - T_{\al}(f^*)}/\epsilon$ and $V^T \eqdef \pa{T_{\be}(g) - T_{\be}(g^*)}/\epsilon$. Then, we can find an upper bound for $ \max \Hat U$,
    \begin{align*}
        \max \Hat U &= -\log \min \exp(-\Hat U)\\
        &= -\log \pa{\min \pa{  (1-\eta_t)\exp(-U) + \eta_t\exp(-V^T+\zeta_{\Hat\be})}} \qquad \text{by the update \eqref{eq1:C-transfer_updates}}\\
        &\leq -\log\pa{ (1-\eta_t)\min \exp(-U) + \eta_t \min\exp(-V^T+\zeta_{\Hat\be}) } \qquad \text{by } \min f_1 + \min f_2\leq \min (f_1+f_2)\\
        &\leq -\pa{ 1-\eta_t}\log\min\exp(-U) - \eta_t\log\min\exp(-V^T+\zeta_{\Hat\be}) \qquad \text{by Jensen Inequality}\\
        &= \pa{ 1-\eta_t}\max U + \eta_t \max \pa{V^T-\zeta_{\Hat\be}},
    \end{align*}
    and similarly,
    \begin{align*}
        \min \Hat U &\geq \pa{ 1-\eta_t}\min U + \eta_t \min \pa{V^T-\zeta_{\Hat\be}},\\
        \max \Hat V &\leq \pa{ 1-\eta_t}\max V + \eta_t \max \pa{\Hat U^T-\iota_{\Hat\al}},\\
        \min \Hat V &\geq \pa{ 1-\eta_t}\min V + \eta_t \min \pa{\Hat U^T-\iota_{\Hat\al}}.
    \end{align*}
    
    Therefore, by the contractivity of the soft-$C$ transform \cite[Proposition 19]{vialard2019elementary} that there exists a contractivity constant $\kappa<1$ such that $\norm{V^T}_{var}\leq \kappa\norm{V}_{var}$ and $\norm{\Hat U^T}_{var}\leq \kappa\norm{\Hat U}_{var}$ for any $t$,
    \begin{align}
        \|\Hat U\|_{var} &\leq (1-\eta_t) \|U\|_{var} + \eta_t \|V^T\|_{var} + \eta_t \|\zeta_{\Hat\be}\|_{var}\notag\\
        &\leq(1-\eta_t) \|U\|_{var} + \eta_t\kappa \|V\|_{var} + \eta_t \|\zeta_{\Hat\be}\|_{var} \label{eq1:iter_bound}\\
        \|\Hat V\|_{var} &\leq (1-\eta_t) \|V\|_{var} + \eta_t \|\Hat U^T\|_{var} + \eta_t \|\iota_{\Hat\al}\|_{var}\notag\\
        &\leq (1-\eta_t) \|V\|_{var} + \eta_t \kappa\|\Hat U\|_{var} + \eta_t \|\iota_{\Hat\al}\|_{var}\label{eq2:iter_bound}.
    \end{align}
    
    Substitute \eqref{eq1:iter_bound} into the RHS of \eqref{eq2:iter_bound},
    \begin{equation}\label{eq3:iter_bound}
        \begin{split}
            \|\Hat V\|_{var} &\leq (1-\eta_t) \|V\|_{var} + \eta_t \kappa\pa{(1-\eta_t) \|U\|_{var} + \eta_t\kappa \|V\|_{var} + \eta_t \|\zeta_{\Hat\be}\|_{var} } + \eta_t \|\iota_{\Hat\al}\|_{var}\\
            &= \pa{1-\eta_t+\eta_t^2\kappa^2} \|V\|_{var} + (1-\eta_t)\eta_t \kappa \|U\|_{var} + \eta_t^2\kappa\|\zeta_{\Hat\be}\|_{var} + \eta_t \|\iota_{\Hat\al}\|_{var}.
        \end{split}
    \end{equation}

    Add up \eqref{eq1:iter_bound} and \eqref{eq3:iter_bound}
    \begin{align*}
        \norm{\Hat U}_{var} +\norm{\Hat V}_{var} \leq & (1-\eta_t + \eta_t \kappa-\eta_t^2 \kappa) \|U\|_{var} + \pa{1 - \eta_t + \eta_t\kappa+\eta_t^2\kappa^2}\|V\|_{var} \\
        &+ \pa{\eta_t +\eta_t^2\kappa}\|\zeta_{\Hat\be}\|_{var} +  \eta_t \|\iota_{\Hat\al}\|_{var}.
    \end{align*}
    
    For $t$ large enough,
    \begin{equation}\label{eq:et_updates}
        \|\Hat U\|_{var} + \|\Hat V\|_{var} \leq \pa{ 1-\eta_t + \eta_t \kappa }\pa{\|U\|_{var} + \|V\|_{var}} +\eta_t\pa{\|\zeta_{\Hat\be}\|_{var} +\|\iota_{\Hat\al}\|_{var}},
    \end{equation}
    which holds for a possibly increased value of $\kappa$, as $\eta_t^2$ is negligible compared to $\eta_t$.
\end{proof}

Making use of a discrete version of Gronwall's lemma \cite[Lemma 5.1]{gronwell-notes}, we are able to show the upper bound of $a_t$ by the recursion relation $a_{t+1}\lesssim\pa{ 1-\eta_t + \eta_t \kappa }a_t+ t^{\theta}$.
\begin{lem}\label{lem:summimg_up_e_t}
    Given a sequence $a_n$ such that $a_{t+1}\lesssim\pa{ 1-\eta_t + \eta_t \kappa }a_t+ t^{\theta}$. Then
    \begin{equation}
        a_{t+1} \lesssim  \pa{a_1+C_1}\exp\pa{C_2 t^{b+1}} + t^{\theta-b},
    \end{equation}
    where $C_1=\frac{1}{-\theta-1}>0$ and $C_2=\frac{\kappa-1}{b+1}<0$
\end{lem}
\begin{proof}
    Summing over $t$, by the discrete version of Gronwall lemma \cite[Lemma 5.1]{gronwell-notes}, we have 
    \begin{align*}
        a_{t+1} \lesssim& \prod_{i=1}^{t} \pa{ 1-\eta_i + \eta_i \kappa } a_1 + \sum_{i=1}^{t-1} \pa{\prod_{j=i+1}^{t}\pa{ 1-\eta_j + \eta_j \kappa }} i^{\theta}+ t^{\theta}\\
        \eqdef& A_{1,t} a_1 + A_{2,t} + t^{\theta},
    \end{align*}
    where 
    \begin{align*}
        A_{1,t} &= \prod_{i=1}^{t} \pa{ 1-\eta_i + \eta_i \kappa },\\
        A_{2,t} &= \sum_{i=1}^{t-1} \pa{\prod_{j=i+1}^{t}\pa{ 1-\eta_j + \eta_j \kappa }} i^{\theta}.
    \end{align*}

    First consider the term $A_{1,t}$, and taking the logarithm on it, using $\log\pa{1+x}<x$,
    \begin{equation}\label{eq:A_{1,t}}
        \begin{split}
            \log A_{1,t} &= \sum_{i=1}^{t}\log \pa{1+\pa{\kappa-1}\eta_i} \leq \sum_{i=1}^t\pa{\kappa-1}\eta_i\\
            &< \pa{\kappa-1}\int_1^{t+1} x^b\mathrm{d}x = \frac{\kappa-1}{b+1} \pa{\pa{t+1}^{b+1}-1}.
        \end{split}
    \end{equation}

    Thus, $A_{1,t} \leq \exp \pa{\frac{\kappa-1}{b+1} \pa{\pa{t+1}^{b+1}-1}}$.
    
    Now, take a look at the term $A_{2,t}$. Following the proof of \cite[Theorem 1]{moulines2011non}, for any $1<m<t-1$,
    \begin{equation}
        A_{2,t} = \sum_{i=1}^{m}\pa{\prod_{j=i+1}^{t}\pa{ 1-\eta_j + \eta_j \kappa }} i^{\theta} + \sum_{i=m+1}^{t-1}\pa{\prod_{j=i+1}^{t}\pa{ 1-\eta_j + \eta_j \kappa }} i^{\theta}.
    \end{equation}

    The first term $\sum_{i=1}^{m}\pa{\prod_{j=i+1}^{t}\pa{ 1-\eta_j + \eta_j \kappa }} i^{\theta} \leq \exp\pa{\sum_{j=m+1}^t\pa{\kappa-1}\eta_j}\sum_{i=1}^m i^{\theta}$, and the second term
    \begin{equation}
        \begin{split}
            \sum_{i=m+1}^{t-1}\pa{\prod_{j=i+1}^{t}\pa{ 1-\eta_j + \eta_j \kappa }} i^{\theta} & \leq m^{\theta-b}\sum_{i=m+1}^{t-1}\prod_{j=i+1}^{t}\pa{ 1-\pa{1-\kappa}\eta_j}\eta_i\\
            & = \frac{m^{\theta-b}}{1-\kappa}\sum_{i=m+1}^{t-1}\left[ \prod_{j=i+1}^t \pa{1-\pa{1-\kappa}\eta_j} - \prod_{j=i}^t \pa{1-\pa{1-\kappa}\eta_j} \right]\\
            &\leq \frac{m^{\theta-b}}{1-\kappa}\left[ 1- \prod_{j=m+1}^{t-1}\pa{1-\pa{1-\kappa}\eta_j} \right] < \frac{m^{\theta-b}}{1-\kappa} .
        \end{split}
    \end{equation}

    Therefore,
    \begin{align*}
        A_{2,t} & < \exp\pa{\pa{\kappa-1}\int_{m+1}^{t+1} x^{b} \mathrm{d}x}\int_{1}^m x^{\theta}\mathrm{d}x +\frac{m^{\theta-b}}{1-\kappa }\\
        &= \exp\pa{\frac{\kappa-1}{b+1}\pa{\pa{t+1}^{b+1}-\pa{m+1}^{b+1}}}\frac{m^{\theta+1}-1}{\theta+1} +\frac{m^{\theta-b}}{1-\kappa }.
    \end{align*}
    
    Take $m=\frac{t}{2}$,
    \begin{equation}\label{eq:A_{2,t}}
        A_{2,t} \lesssim \exp\pa{\frac{\kappa-1}{b+1} \pa{\pa{t+1}^{b+1}-\pa{t/2}^{b+1}}}\frac{1}{-\theta-1} +\frac{t^{\theta-b}}{1-\kappa }.
    \end{equation}

    Combine Equations \eqref{eq:A_{1,t}} and \eqref{eq:A_{2,t}}
    \begin{equation}
        a_{t+1} \lesssim  \pa{a_1+\frac{1}{-\theta-1}}\exp\pa{\frac{\kappa-1}{b+1} t^{b+1}} + t^{\theta-b}.
    \end{equation}
\end{proof}

\subsection{Proof of Theorem \ref{thm:OSerror}}

\begin{proof}
    Let $U_t \eqdef \pa{ f_t - f^*}/\epsilon$, $V_t \eqdef \pa{g_t - g^*}/\epsilon$, $e_t \eqdef \|U_t\|_{var} + \|V_t\|_{var}$, $\zeta_{\Hat\be_t} \eqdef \pa{T_{\be}(g_t) - T_{\Hat\be_t}(g_t)}/\epsilon$ and $\iota_{\Hat\al_t}\eqdef \pa{T_\al(f_{t+1})- T_{\Hat\al_t}(f_{t+1})}/\epsilon$. By Lemma \ref{lem:e_{t+1}_and _e_t},
    \begin{equation}\label{eq:error_updates}
        e_{t+1} \leq \pa{ 1-\eta_t + \eta_t \kappa }e_t +\eta_t\pa{\|\zeta_{\be_t}\|_{var} +\|\iota_{\al_t}\|_{var}}.
    \end{equation}
    
    Define $M_t=\int\exp((f^*+g_t-c)/\epsilon)\mathrm{d}\be$, then
    \begin{align*}
        M_t&=\int\exp((f^*+g^*-g^*+g_t-c)/\epsilon)\mathrm{d}\be\\
        &= \int\exp((-g^*+g_t)/\epsilon)\mathrm{d}\be,
    \end{align*}
    and by Lemma \ref{lemma_bound}, there exists $\delta = \sup_{x\in\Xx}|f_{t_0}-f^*|>0$ for $t_0<t$, such that $ \exp(-\delta/\epsilon) <M_t<\exp(\delta/\epsilon)$. Apply Lemma \ref{lem:error_term_bound}, 
    \begin{equation}\label{eqs:uniform_law_of_large_numbers}
        \begin{split}
            \mathbb{E}\norm{\zeta_{\Hat\be_t}}_{var}&\lesssim \frac{A_1(f^*,g^*,\epsilon)}{\sqrt{b_t}},\\
        \mathbb{E}\norm{\iota_{\Hat\al_t}}_{var}&\lesssim \frac{A_2(f^*,g^*,\epsilon)}{\sqrt{b_t}},
        \end{split}
    \end{equation}
    where $A_1(f^*,g^*,\epsilon),A_2(f^*,g^*,\epsilon)$ are constants depending on $f^*,g^*,\epsilon$.
    
    Denote $S= A_1(f^*,g^*,\epsilon)+A_2(f^*,g^*,\epsilon)$, then taking expectation on \eqref{eq:error_updates} with Equations \eqref{eqs:uniform_law_of_large_numbers}
    \begin{equation}\label{error_recursion}
        \begin{split}
            \mathbb{E}e_{t+1} &\leq \pa{ 1-\eta_t + \eta_t \kappa }\mathbb{E}e_t +\eta_t\mathbb{E}\pa{\|\zeta_{\be_t}\|_{var} +\|\iota_{\al_t}\|_{var}}\\
        &\lesssim \pa{ 1-\eta_t + \eta_t \kappa }\mathbb{E} e_t+ \frac{S\eta_t}{\sqrt{b_t}}.
        \end{split}
    \end{equation}
    
    By Lemma \ref{lem:summimg_up_e_t}, we have
    \begin{equation}\label{delta_upperbound}
        \mathbb{E} e_{t+1} \lesssim \pa{\mathbb{E} e_1 + \frac{S}{a-b-1}}\exp\pa{\frac{\kappa-1}{b+1} t^{b+1}} + t^{-a},
    \end{equation}
    whose RHS converges to $0$ as $t\to\infty$
    
    Consider the upper bound \eqref{delta_upperbound} for $\delta_{t+1}$, 
    \[
    \delta_{t+1} \lesssim \pa{\delta_1 + \frac{S}{a-b-1}}\exp\pa{\frac{\kappa-1}{b+1} t^{b+1}} + t^{-a}.
    \]
    
    Note that the total sample size at step $t$ is $n_t = \sum_{i=1}^t i^{2a}=O\pa{t^{2a+1}}$, thus we can rewrite $t$ in terms of $n_t$, that is $t = n_t^{\frac{1}{2a+1}}$.
    The first term decays rapidly with the rate ${O}\pa{\exp\pa{-n^{\frac{b+1}{2a+1}}}}$, since $\frac{\kappa-1}{b+1}<0$. 
    Taking $t(N) =  \lfloor N^{\frac{1}{2a+1}} \rfloor$, then when $t^{\frac{b+1}{2a+1}}\gg1$,
    \[
        \delta_N \lesssim  \pa{\delta_1 + \frac{S}{a-b-1}}\exp\pa{\frac{\kappa-1}{b+1} N^\frac{b+1}{2a+1}}+N^{-\frac{a}{2a+1}} = {O}\pa{N^{-\frac{1}{2+1/a}}},
    \]
    which is bounded by $O\pa{N^{-1/2}}$.
\end{proof}

\subsection{Proof of Theorem \ref{thm:COSerror}}

\begin{proof}
    This proof follows up on the proof of Theorem \ref{thm:OSerror}.
    
    Recall the following terms regarding $f_t$ and $g_t$,
    \begin{align*}
        & U_t \eqdef \pa{ f_t - f^*}/\epsilon, & V_t &\eqdef \pa{g_t - g^*}/\epsilon,\\
        & U_t^T \eqdef \pa{T_{\al}(f_t) - T_{\al}(f^*)}/\epsilon, & V_t^T &\eqdef \pa{T_{\be}(g_t) - T_{\be}(g^*)}/\epsilon,
    \end{align*}
    and further define the corresponding terms regarding $\Hat f_t$ and $\Hat g_t$,
    \begin{align*}
        &\Hat U_t \eqdef \pa{ \Hat f_t - f^*}/\epsilon, & \Hat V_t &\eqdef \pa{\Hat g_t - g^*}/\epsilon,\\
        &\Hat U_t^T \eqdef \pa{T_{\al}(\Hat f_t) - T_{\al}(f^*)}/\epsilon, & \Hat V_t^T &\eqdef \pa{T_{\be}(\Hat g_t) - T_{\be}(g^*)}/\epsilon,\\
        &\Hat\zeta_{\Hat\be_t} \eqdef \pa{T_{\be}(\Hat g_t) - T_{\Hat\be_t}(\Hat g_t)}/\epsilon, & \Hat \iota_{\Hat\al_t}&\eqdef \pa{T_\al(f_{t+1})- T_{\Hat\al_t}(f_{t+1})}/\epsilon.
    \end{align*}

    Notice that under Assumption \ref{assumption4}, $\norm{\pa{f_t-\Hat{f}_t}/\epsilon}_{var}= O\pa{t^{-a+b}}$. Suppose that $$\max\ens{\sup_{x}\abs{\Hat f_t-f^*},\sup_{y} \abs{\Hat g_t-g^*}}<\delta, $$ then by Lemma \ref{lemma_bound}, $\underset{x}{\sup}\abs{ f_{t+1} - f^*} < \delta$. Thus,
    \begin{align*}
        \underset{x}{\sup}\abs{\Hat f_{t+1} - f^*} &= \underset{x}{\sup}\abs{\Hat f_{t+1} -f_{t+1} + f_{t+1} - f^*}\\
        &\leq \underset{x}{\sup}\abs{ f_{t+1} - f^*} + \underset{x}{\sup}\abs{\Hat f_{t+1} -f_{t+1}}\\
        &\lesssim \underset{x}{\sup}\abs{ f_{t+1} - f^*} + t^{-a+b} < \delta +  t^{-a+b}.
    \end{align*}

    Let $\delta = \max\ens{\underset{x}{\sup}\abs{ \Hat f_{1} - f^*} ,\underset{y}{\sup}\abs{ \Hat g_{1} - g^*}}$, we have 
    \begin{equation*}
        \underset{x}{\sup}\abs{\Hat f_{t+1} - f^*} \leq \delta + \sum_{i=1}^t i^{-a+b} ,
    \end{equation*}
    which is bounded from above and below as $-a+b+1<0$. 
    
    Notice that $\Hat g_t$ and $f_{t+1}$ are $\epsilon L$-Lipschitz, and apply Lemma \ref{lem:error_term_bound}, 
    \begin{align*}
        \mathbb{E}\norm{\zeta_{\Hat\be_t}}_{var}&\lesssim \frac{A_1(f^*,g^*,\epsilon)}{\sqrt{b_t}},\\
        \mathbb{E}\norm{\iota_{\Hat\al_t}}_{var}&\lesssim \frac{A_2(f^*,g^*,\epsilon)}{\sqrt{b_t}},
    \end{align*}
    where $A_1(f^*,g^*,\epsilon),A_2(f^*,g^*,\epsilon)$ are constants depending on $f^*,g^*,\epsilon$.
    
    We have the following relation
    \begin{align*}
        \Hat U_{t+1} &= \pa{ \Hat f_{t+1} - f_{t+1}}/\epsilon + \pa{ f_{t+1} - f^*}/\epsilon \eqdef \text{err}_{f_{t+1}} + u_{t+1},\\
        \Hat V_{t+1} &= \pa{ \Hat g_{t+1} - g_{t+1}}/\epsilon + \pa{ g_{t+1} - g^*}/\epsilon \eqdef \text{err}_{g_{t+1}} + V_{t+1},
    \end{align*}
    where $\text{err}_{f_{t+1}} = \pa{ \Hat f_{t+1} - f_{t+1}}/\epsilon$ and $\text{err}_{g_{t+1}} = \pa{ \Hat g_{t+1} - g_{t+1}}/\epsilon$, and
    \begin{align}\label{ineq:Hat_u_and_u}
        \norm{\Hat U_{t+1}}_{var} &\leq \norm{\text{err}_{f_{t+1}}}_{var} + \norm{U_{t+1}}_{var},\\
        \norm{\Hat V_{t+1}}_{var} &\leq \norm{\text{err}_{g_{t+1}}}_{var} + \norm{V_{t+1}}_{var}.
    \end{align}
    
    Recall from Theorem \ref{thm:OSerror} that $e_t \eqdef \|U_t\|_{var} + \|V_t\|_{var}$, and define $\Hat e_t \eqdef \|\Hat U_t\|_{var} + \|\Hat V_t\|_{var}$. By Lemma \ref{lem:e_{t+1}_and _e_t}, for $t$ large enough
    \begin{equation}\label{ineq:et&etHat}
        e_{t+1} \leq \pa{ 1-\eta_t + \eta_t \kappa } \Hat e_{t} + \eta_t \pa{\|\Hat \zeta_{\Hat\be_t}\|_{var}+ \|\Hat\iota_{\Hat\al_t}\|_{var}}.
    \end{equation}

    Thus, by the inequalities \eqref{ineq:Hat_u_and_u} and \eqref{ineq:et&etHat},
    \begin{equation}\label{ineq:e_t_iterations}
        \begin{split}
            \Hat e_{t+1} &\leq e_{t+1} + \norm{\text{err}_{f_{t+1}}}_{var} + \norm{\text{err}_{g_{t+1}}}_{var}\\
        &\leq \pa{ 1-\eta_t + \eta_t \kappa } \Hat e_{t} + \eta_t \pa{\|\Hat \zeta_{\Hat\be_t}\|_{var}+ \|\Hat\iota_{\Hat\al_t}\|_{var}} + \norm{\text{err}_{f_{t+1}}}_{var} + \norm{\text{err}_{g_{t+1}}}_{var}
        \end{split}
    \end{equation}
    
    Take expectations on both sides of \eqref{ineq:e_t_iterations}, we have
    \begin{align*}
        \mathbb{E}  \Hat e_{t+1}  &\leq \mathbb{E} e_{t+1} + \norm{\text{err}_{f_{t+1}}}_{var} + \mathbb{E}\norm{\text{err}_{g_{t+1}}}_{var}\\
        &\lesssim  \pa{ 1-\eta_t + \eta_t \kappa } \mathbb{E} \Hat e_t +  \frac{S\eta_t}{\sqrt{b_t}} + \mathbb{E}\norm{\text{err}_{f_{t+1}}}_{var} + \mathbb{E}\norm{\text{err}_{g_{t+1}}}_{var},
    \end{align*}
    where $S=A_1(f^*,g^*,\epsilon)+A_2(f^*,g^*,\epsilon)$.
    
    Notice that $\norm{\text{err}_{f_{t+1}}}_{var}=O\pa{t^{-a+b}}$ under Assumptions \ref{assumption4} and \ref{assumption5}. By Lemma \ref{lem:summimg_up_e_t} with $\theta=b-a$,
    \begin{equation}
        \mathbb{E} \Hat e_{t+1} \lesssim  \pa{\mathbb{E}e_1 + \frac{1}{-\theta-1}} \exp\pa{\frac{\kappa-1}{b+1} t^{b+1}}+ \frac{1}{\pa{1-\kappa}} t^{-a},
    \end{equation}
    where the right-hand side goes to $0$ as $t\to\infty$.
\end{proof}

\subsection{Compression errors}\label{subsec:compression_err}
\subsubsection{Gauss quadrature (GQ)}
A quadrature rule uses a sum of specific points with assigned weights as an approximation to an integral, which are optimal with respect to a certain polynomial degree of exactness \cite{gautschi2004orthogonal}. The $m$-point Gauss quadrature rule for $\mu$ can be expressed as
\begin{equation*}
    \int_{\mathbb{R}} f\pa{y} \mathrm{d}\mu (y) = \sum_{i=1}^m w_i f\pa{\Hat y_i} + R_m \pa{f},
\end{equation*}
for weights $w_i$ and points $\hat y_i$,
where the remainder term $R_m$ satisifes $R_m \pa{\mathbb{P}_{2m-1}}=0$, and therefore the sum approximation on the RHS matches the integral value on the LHS for $f\in \mathbb{P}_{2m-1}$.

In general, by \cite[Corollary to Theorem 1.48]{gautschi2004orthogonal}, the error term $R_m$ can be expressed as 
\[
    R_m(f) = M\frac{f^{(2m)}(\xi)}{(2m)!}, \quad\text{some $\xi\in\mathbb R$}
\]
where $M=\int_{\mathbb{R}}\left[\pi_n\pa{t;d\mu}\right]^2\mathrm{d}\mu\pa{t}$ and $\pi_n\pa{\cdot;d\mu}$ is the numerator polynomial \cite[Definition 1.35]{gautschi2004orthogonal}. 

In our case, $f=K_x/v_t$.
Note that $K_x$ and $v_t$ are smooth (both have the same regularity as $e^{-C(x,y)/\epsilon}$; see \eqref{eq:Sinkhorn_weights_representation}). %{ \color{green} 
Moreover, note that $v_t$ is uniformly bounded away from 0 since $g_t$ is uniformly bounded. %}
For $K_x(\xi) = \exp\pa{-\frac{C\pa{x,\xi}}{\epsilon}}$ and, by the closed form of Gaussian functions,
\begin{align*}
    K_x^{(2m)}(\xi) = \pa{\frac{1}{\epsilon}}^m\exp\pa{-\frac{\pa{x-\xi}^2}{\epsilon}} H_{2m}\pa{\frac{x-\xi}{\sqrt{\epsilon}}}=O\pa{\frac{\pa{2m}!}{m!\epsilon^m}},
\end{align*}
where $H_n$ is the Hermite polynomial of $n$th order.  %{\color{green}
It follows, by the Leibniz rule and the Faa di Bruno formula that, %}
$R_m\pa{f} =O\pa{ \frac{1}{\epsilon^m m!}}$.  Hence,
\begin{align*}
u_t(x)-\hat{u}_t(x)& =
\int \frac{K_x}{v_t}(y) \mathrm{d}\mu(y)-\int \frac{K_x}{v_t} (y) \mathrm{d}\Hat\mu (y) =O\pa{ \frac{1}{\epsilon^m m!}}.
\end{align*}
Note that $f_t(x)=-\epsilon\log(u_t(x))$. By the mean-value theorem, $|\log a-\log b| \le \frac{1}{a} |a-b|$ for $0<a<b$.  Further, $u_t$ is bounded away from zero (as a continuous and positive function on a compact set). Hence, we may find a Lipschitz constant $L'$ such that $|f_t(x)-\hat f_t(x)| \le L' |u_t(x)-\hat u_t(x)|$ for all $x\in \mathcal X$. Hence,  Assumption \ref{assumption4} holds for any $\zeta>0$.
%\todo{fix the $J_x$ above.}

\subsubsection{Fourier method}\label{subsec:Fourier}
Consider $C(x,y)=\|x-y\|^2$. Take Fourier moments of $\int \frac{K_{x}}{v_t}\pa{y}\mathrm{d}\mu\pa{y}$,
\begin{equation}\label{eq:Fourier_moment}
    \int \exp\pa{-ikx}  \int \frac{K_x}{v_t}\pa{y} d \mu\pa{y} dx = \int \frac{\Hat K_y(k)}{v_t(y)} \mathrm{d}\mu\pa{y}= \sqrt{\epsilon\pi}\exp\pa{-\frac{\epsilon k^2}{4}} \int \frac{\exp\pa{-iky}}{v_t(y)}\mathrm{d}\mu\pa{y},
\end{equation}
as $\Hat K_y(k)$ for $K_x\pa{y}=\exp\pa{-\frac{\pa{x-y}^2}{\epsilon}}$ is given by 
\begin{equation}
\begin{split}
    \Hat K_y(k) &=   \int \exp\pa{-ikx} \exp\pa{-\frac{\pa{x-y}^2}{\epsilon}} dx\\
    &= \exp\pa{-iky}\int \exp\pa{-ikz} \exp\pa{-\frac{z^2}{\epsilon}} dz\\
    &= \exp\pa{-iky} \sqrt{\epsilon\pi}\exp\pa{-\frac{\epsilon k^2}{4}}.
\end{split}
\end{equation}

By the Fourier inversion theorem,
\begin{equation}
 u_t(x)=   \exp\pa{-f_t\pa{x}/\epsilon} = \int \exp\pa{ikx} \sqrt{\epsilon\pi}\exp\pa{-\frac{\epsilon k^2}{4}} \int \frac{\exp\pa{-iky}}{v_t(y)}\mathrm{d}\mu\pa{y}\mathrm{d} k.
\end{equation}
The compression error becomes
\begin{equation}
 u_t(x)-\hat u_t(x)=   \int \frac{K_x}{v_t}(y) d\mu(y)- \int \frac{K_x}{v_t}(y) d\hat \mu(y) = \int \frac{\phi_x\pa{z}}{v_t\pa{y}} \mathrm{d} z,
\end{equation}
where
\begin{equation}
    \phi_x\pa{k} =\exp\pa{i\,z\,k} \sqrt{\epsilon\pi}\exp\pa{-\frac{\epsilon k^2}{4}} \int \exp\pa{-i\,k\,y}\mathrm{d}\pa{\mu-\Hat\mu}\pa{y},
\end{equation}
and note that, by construction, $\phi_x\pa{k} = 0$ for all $k\in \Omega$.
The problem of finding the compression thus becomes a problem of evaluating the integral $\int \phi_x\pa{z} \mathrm{d}z$. Similarly to Gauss quadrature, we want to find $k\in \Omega=\ens{k_1,\cdots,k_m}$ such that the $\frac1m\sum_{i=1}^{m} \phi_x\pa{k_i}$ is an approximation to $\int \phi_x\pa{z} \mathrm{d}z$.  
% {
% \color{green}Is the following paragraph here a repeat/fleshing out of details for eqn (56) above? Or should we move the complexity discussion in the previous paragraph elsewhere?}

Let $\Omega=\ens{k_1,\cdots,k_m}$ be the set of $n$ elements that are QMC sampled from $X\sim \Nn\pa{0,\frac{2}{\epsilon}I}$. In practice, we use the implementation of \texttt{SciPy} \cite{2020SciPy-NMeth}.
Define $\chi\pa{z} = \exp\pa{\frac{\epsilon z^2}{4}} \phi_x\pa{z}$. Notice that
\begin{equation}
    \begin{split}
        \mathbb{E}\pa{\chi\pa{X}} &= \int_{\mathbb{R}^d}\pa{2\pi}^{-d/2}\pa{2/\epsilon}^{-d/2}\exp\pa{-\frac{\epsilon}{4}z^2}\psi\pa{z}\mathrm{d}z\\
        &= \pa{\frac{4\pi}{\epsilon}}^{-d/2} \int_{\mathbb{R}^d} \phi_x\pa{z}\mathrm{d}z,
    \end{split}
\end{equation}
and thus,
\begin{equation}
    \int_{\mathbb{R}^d} \phi_x\pa{z} \mathrm{d}z = \pa{\frac{4\pi}{\epsilon}}^{d/2} \mathbb{E} \chi\pa{X}.
\end{equation}

Following \cite[Section 4.1]{kuo2016practical}, let $\Psi(x)$ denote the cumulative distribution functions for $\mathcal{N}(0,1)$. Then $\Psi(\sqrt{\frac{\epsilon}{{2}}}X_i)\sim\mathcal{U}(0,1)$ and 
\begin{equation}    \mathbb{E}\pa{\chi\pa{X}}=\int_{[0,1]^d} \chi \pa{A\Psi^{-1}\pa{z}}\mathrm{d}z,
\end{equation}
where $A=\sqrt{\frac{2}{\epsilon}}I$, $\Psi^{-1}\pa{z}=\pa{\Psi^{-1}\pa{z^1},\cdots,\Psi^{-1}\pa{z^d}}$.
Denote $h\pa{z}= \chi \pa{A\Psi^{-1}\pa{z}}$, and $I_d\pa{h} = \int_{[0,1]^d}h\pa{z}\mathrm{d}t$, where $d$ is the dimension, and $Q_{m,d}\pa{h}=\frac1m\sum_{i=1}^m h\pa{k_i}$. 
%Let $V\pa{g}$ be the variation \cite[Definition 5.2]{kuipers2012uniform}, then following \cite[Example 2.4]{dick2013high}.
% \begin{equation}
%     \abs{I_d\pa{g}-Q_{n,d}\pa{g}} \leq C\pa{d}V\pa{g}\frac{\abs{\log n}^d}{n}.
% \end{equation}
As $y\in\mathcal X$, a compact domain, the first derivative of $\chi(k)$ is bounded and the integrand $h$ has bounded variation.
Hence, by \cite[Section 1.3]{kuo2016practical}, 
\begin{equation}
    \abs{I_d\pa{h}-Q_{m,d}\pa{h}} = O\pa{\frac{\abs{\log m}^d}{m}}.
\end{equation}

Given that for $k\in \Omega$, $\phi_x\pa{k}=0$ and hence, $Q_{m,d}\pa{h}=0$. 
Then, we know
\begin{equation}
\begin{split}
\int_{\mathbb{R}^d} \phi_x\pa{z} \mathrm{d}z= \pa{\frac{4\pi}{\epsilon}}^{d/2} I_d\pa{h} &\leq \pa{\frac{4\pi}{\epsilon}}^{d/2}\pa{Q_{m,d}\pa{h} + \abs{I_d\pa{h}-Q_{m,d}\pa{h}}}\\
&= O\pa{\pa{\frac{4\pi}{\epsilon}}^{d/2}\frac{\abs{\log m}^d}{m}}.
\end{split}
\end{equation}

Choosing $k$ in \eqref{eq:moment} via Gaussian QMC, the compression error is then
\begin{equation}
u_t(x)-\hat u_t(x)=    \exp\pa{-f_t\pa{x}/\epsilon} - \exp\pa{-\Hat{f}_t\pa{x}/\epsilon} =  O\pa{\frac{\abs{\log m}^d}{m}}.
\end{equation}

If we take $\zeta<1$, then $\frac{\abs{\log m}^d}{m}=o\pa{m^{-\zeta}}$. Again applying the Lipschitz argument, $f_t(x)-\hat f_t(x)=O\pa{\frac{\abs{\log m}^d}{m}}$ and  Assumption~\ref{assumption4} holds.
% {\color{green}
% How do you bound $\norm{f_t - \hat f_t}$ from the above? Are you just using the mean value theorem applied to $\log(u_t)$
% }

% Using this method, compressing an $n_t$-point discrete measure to an $m$-point measure consists of solving the linear system $\int \frac{\hat K_y(k)}{v_t(y)}\mathrm{d}\pa{\mu-\Hat\mu}\pa{y}=0$. The complexity of calculating the matrix $M_t = \pa{ P_k(y_i)}_{k\in\Omega, i\in [n]}$ is $O\pa{dmn}$, and for solving the linear system is $O\pa{m^3}$. Therefore, the total complexity for the compression method is $O\pa{m^3+dmn}$.

\subsection{Proof of Proposition \ref{prop:compression_complexity}}
\begin{proof}
    \textbf{Cost of online Sinkhorn}
    Now we calculate the complexity of Algorithm \ref{alg:compressed_online_sinkhorn} up to step $T$. Recalled that in Section \ref{sec:OS_complexity}, the computational complexity of Algorithm \ref{alg:online_sinkhorn} is $\Cc = O\pa{d T^{4a+2}}$. By Theorem \ref{thm:OSerror}, the error at step $T$ is $\delta = O\pa{T^{-a}}$ (where the hidden constant may depend on dimension). Taking $T=O(\delta^{-1/a})$, we have
    $\Cc = O\pa{ \delta^{-\pa{4+\frac2a}}}$.

    \textbf{Cost of compressed online Sinkhorn}
    
  Assuming$m_t=t^{\frac{a-b}{\zeta}}$,
    the total computational cost of Algorithm \ref{alg:compressed_online_sinkhorn} is
    \begin{align*}
        \Hat{\Cc} &= \sum_{t=1}^T  \pa{C(n_{t-1},m_t) + d b_t m_t}\\
        &= \sum_{t=1}^T  \pa{m_t^{3} + d \pa{m_{t-1}+b_{t-1}}  m_t + d b_t m_t} \\
        &= O\pa{\sum_{t=1}^T  \pa{t^{\frac{3\pa{a-b}}{\zeta}} +d \pa{t-1}^{2a}  t^{\frac{a-b}{\zeta}} + d t^{2a+\frac{a-b}{\zeta}}}}\\
        &=\begin{cases} O\pa{d T^{2a+\frac{a-b}{\zeta}+1}}, & \text{ for } \zeta\geq\frac{a-b}{a}\\ O\pa{T^{\frac{3\pa{a-b}}{\zeta}+1}}, & \text{ for } \zeta< \frac{a-b}{a} \end{cases}.
    \end{align*}

    Under Assumptions \ref{assumption4} and \ref{assumption5}, the error at step $T$ is $\delta = O\pa{T^{-a}}$. Taking $T=\delta^{-1/a}$, we have
    $$
    \Hat{\Cc} = \begin{cases} O\pa{\delta^{-\pa{2+\frac{a-b}{a \zeta}+\frac{1}{a}}}}, & \text{ for } \zeta\geq\frac{a-b}{a}, \\ O\pa{\delta^{-\pa{\frac{3\pa{a-b}}{a\zeta}+\frac1a}}}, & \text{ for } \zeta< \frac{a-b}{a}. \end{cases}
    $$
    The big-$O$ constant for the accuracy $\delta$ may depend on dimension.
    \end{proof}

Notice that the ratio of $\Hat\Cc$ and $\Cc$,
\begin{equation}
    \frac{\Hat \Cc}{\Cc}  = \begin{cases} \Oo\pa{\delta^{2+\frac1a-\frac{a-b}{a \zeta}}}, & \text{ for } \zeta\geq\frac{a-b}{a}, \\ 
    \Oo\pa{\delta^{4+\frac1a-\frac{3\pa{a-b}}{a\zeta}}}, & \text{ for } \zeta< \frac{a-b}{a}, \end{cases}
\end{equation}
and the exponent of the ratio is positive when $\zeta > \frac{3\pa{a-b}}{4a+1}$, which means the asymptotical convergence of the Compressed Online Sinkhorn is improved compared to the Online Sinkhorn.

\end{document}